\documentclass{article}
\usepackage{ijcai18}

\usepackage{times}
\usepackage{xcolor}
\usepackage{soul}
\usepackage[utf8]{inputenc}
\usepackage[small]{caption}

\usepackage{amsmath}
\usepackage{amsthm}
\usepackage{amssymb}
\usepackage{bbm}
\usepackage{algorithm}
\usepackage{algorithmic}
\usepackage{graphicx}
\usepackage{caption}
\usepackage{subcaption}
\usepackage{multirow}
\usepackage[flushleft]{threeparttable}

\newtheorem{theorem}{Theorem}
\newtheorem{corollary}{Corollary}
\newtheorem{remark}{Remark}
\newtheorem{definition}{Definition}
\newtheorem{proposition}{Proposition}
\newtheorem{lemma}{Lemma}
\newcommand{\bq}{\boldsymbol{q}}
\newcommand{\bp}{\boldsymbol{p}}
\newcommand{\balpha}{\boldsymbol{\alpha}}
\usepackage{color}
\newcommand{\blue}{}

\title{UCBoost: A Boosting Approach to Tame Complexity and Optimality for Stochastic Bandits}

\author{
Fang Liu$^1$, 
Sinong Wang$^1$, 
Swapna Buccapatnam$^2$,
Ness Shroff$^1$ 
\\ 
$^1$ The Ohio State University \\
$^2$ AT\&T Labs Research\\
liu.3977@osu.edu,
wang.7691@osu.edu,
sb646f@att.com,
shroff.11@osu.edu
}

\begin{document}

\maketitle

\begin{abstract}
In this work, we address the open problem of finding low-complexity near-optimal multi-armed bandit algorithms for sequential decision making problems. Existing bandit algorithms are either sub-optimal and computationally simple (e.g., UCB1) or optimal and computationally complex (e.g., kl-UCB). We propose a boosting approach to Upper Confidence Bound based algorithms for stochastic bandits, that we call UCBoost. Specifically, we propose two types of UCBoost algorithms. We show that UCBoost($D$) enjoys $O(1)$ complexity for each arm per round as well as regret guarantee that is $1/e$-close to that of the kl-UCB algorithm. We propose an approximation-based UCBoost algorithm, UCBoost($\epsilon$), that enjoys a regret guarantee $\epsilon$-close to that of kl-UCB as well as $O(\log(1/\epsilon))$ complexity for each arm per round. Hence, our algorithms provide practitioners a practical way to trade optimality with computational complexity. Finally, we present numerical results which show that UCBoost($\epsilon$) can achieve the same regret performance as the standard kl-UCB while incurring only $1\%$ of the computational cost of kl-UCB.
\end{abstract}

\section{Introduction}\label{sec:intro}
\begin{table*}[t]
\centering
\caption{Regret guarantee and computational complexity per arm per round of various algorithms}
\label{table:complexityintro}
\begin{center}
\begin{tabular}{|c|c|c|c|c|}
\hline
& kl-UCB & UCBoost($\epsilon$) & UCBoost($D$) &UCB1\\
\hline
Regret/$\log(T)$ & $O\left(\sum\limits_{a}\frac{\mu^*-\mu_a}{d_{kl}(\mu_a,\mu^*)}\right)$ &$O\left(\sum\limits_{a}\frac{\mu^*-\mu_a}{d_{kl}(\mu_a,\mu^*)-\epsilon}\right)$ & $O\left(\sum\limits_{a}\frac{\mu^*-\mu_a}{d_{kl}(\mu_a,\mu^*)-1/e}\right)$ & $O\left(\sum\limits_{a}\frac{\mu^*-\mu_a}{2(\mu^*-\mu_a)^2}\right)$\\
\hline
Complexity & unbounded & $O(\log(1/\epsilon))$ & $O(1)$ & O(1)\\
\hline
\end{tabular}
\end{center}
\end{table*}

Multi-armed bandits, introduced by \citeauthor{thompson}~\shortcite{thompson}, have been used as quintessential models for sequential decision making. In the classical setting, at each time, a decision maker must choose an arm from a set of $K$ arms with unknown probability distributions. Choosing an arm $i$ at time $t$ reveals a random reward $X_{i}(t)$ drawn from the probability distribution of arm $i.$  The goal is to find policies that minimize the expected regret due to uncertainty about arms' distributions over a given time horizon $T$. \citeauthor{lai1985asymptotically}~\shortcite{lai1985asymptotically}, followed by \citeauthor{burnetas1996optimal}~\shortcite{burnetas1996optimal}, have provided an asymptotically lower bound on the expected regret. 

Upper confidence bounds (UCB) based algorithms are an important class of bandit algorithms. The most celebrated UCB-type algorithm is UCB1 proposed by \citeauthor{auer2002finite}~\shortcite{auer2002finite}, which enjoys simple computations per round as well as $O(\log T)$ regret guarantee. Variants of UCB1, such as UCB-V proposed by \citeauthor{audibert2009exploration}~\shortcite{audibert2009exploration} and MOSS proposed by \citeauthor{audibert2010regret}~\shortcite{audibert2010regret}, have been studied and shown improvements on the regret guarantees. However, the regret guarantees of these algorithms have unbounded gaps to the lower bound. Recently, \citeauthor{maillard2011finite}~\shortcite{maillard2011finite} and \citeauthor{garivier2011kl}~\shortcite{garivier2011kl} have proposed a UCB algorithm based on the Kullback-Leibler divergence, kl-UCB, and proven it to be asymptotically optimal when all arms follow a Bernoulli distribution, i.e., they reach the lower bound by \citeauthor{lai1985asymptotically}~\shortcite{lai1985asymptotically}. They have generalized the algorithm to KL-UCB \cite{cappe2013kullback}, which is asymptotically optimal under general distributions with bounded supports.

However, these UCB algorithms exhibit a complexity-optimality dilemma in the real world applications that are computationally sensitive. On one hand, the UCB1 algorithm enjoys closed-form updates per round while its regret gap to the lower bound can be unbounded. On the other hand, the kl-UCB algorithm is asymptotically optimal but it needs to solve a convex optimization problem for each arm at each round. Though there are many standard optimization tools to solve the convex optimization problem numerically, there is no regret guarantee for the implemented kl-UCB with arbitrary numerical accuracy. Practitioners usually set a sufficient accuracy (for example, $10^{-5}$) so that the behaviors of the implemented kl-UCB converge to the theory. However, this means that the computational cost per round by kl-UCB can be out of budget for applications with a large number of arms. The complexity-optimality dilemma is because there is currently no available algorithm that can trade-off between complexity and optimality. 
 
Such a dilemma occurs in a number of applications with a large $K$. For example, in an online recommendation system \cite{li2010contextual,buccapatnam2017reward}, the algorithm needs to recommend an item from hundreds of thousands of items to a customer within a second. Another example is the use of bandit algorithms as a meta-algorithm for other machine learning problems, e.g., using bandits for classifier boosting \cite{busa2010fast}. The number of data points and features can be large.

Another scenario that the dilemma appears is in real-time applications such as robotic systems \cite{matikainen2013multi}, 2D planning \cite{laskey2015multi} and portfolio optimization \cite{moeini2016portfolio}. In these applications, a delayed decision may turn out to be catastrophic.

\citeauthor{cappe2013kullback}~\shortcite{cappe2013kullback} proposed the open problem of finding a low-complexity optimal UCB algorithm, which has remained open till now. In this work, we make the following contributions to this open problem. (Table \ref{table:complexityintro} summarizes the main results.)
\begin{itemize}
\item We propose a generic UCB algorithm. By plugging a semi-distance function, one can obtain a specific UCB algorithm with regret guarantee (Theorem \ref{theorem:UCB}). As a by-product, we propose two new UCB algorithms that are alternatives to UCB1 (Corollary \ref{cor:bq} and \ref{cor:h}).
\item We propose a boosting algorithm, UCBoost, which can obtain a strong (i.e., with regret guarantee close to the lower bound) UCB algorithm from a set of weak (i.e., with regret guarantee far away from the lower bound) generic UCB algorithms (Theorem \ref{theorem:UCBoost}). By boosting a finite number of weak generic UCB algorithms, we find a UCBoost algorithm that enjoys the same complexity as UCB1 as well as a regret guarantee that is $1/e$-close to the kl-UCB algorithm (Corollary \ref{cor:ucboost})\footnote{Note that $e$ is the natural number}. That is to say, such a UCBoost algorithm is low-complexity and near-optimal under the Bernoulli case.
\item We propose an approximation-based UCBoost algorithm, UCBoost($\epsilon$), that enjoys $\epsilon$-optimal regret guarantee under the Bernoulli case and $O(\log(1/\epsilon))$ computational complexity for each arm at each round for any $\epsilon>0$ (Theorem \ref{theorem:thmstepapprox}). This algorithm provides a non-trivial trade-off between complexity and optimality. 
\end{itemize}

{\blue
{\bf \noindent Related Work.} There are other asymptotically optimal algorithms, such as Thompson Sampling~\cite{agrawal2012analysis}, Bayes-UCB~\cite{kaufmann2012bayesian} and DMED~\cite{honda2010asymptotically}. However, the computations involved in these algorithms become non-trivial in non-Bernoulli cases. First, Bayesian methods, including Thompson Sampling, Information Directed Sampling~\cite{russo2014learning,liu2017information} and Bayes-UCB, require updating and sampling from the posterior distribution, which is computationally difficult for models other than exponential families~\cite{korda2013thompson}. Second, the computational complexity of DMED policy is larger than UCB policies because the computation involved in DMED is formulated as a univariate convex optimization problem. In contrast, our algorithms are computationally efficient in general bounded support models and don't need the knowledge of prior information on the distributions of the arms.

Our work is also related to DMED-M proposed by \citeauthor{honda2012stochastic}~\shortcite{honda2012stochastic}. DMED-M uses the first $d$ empirical moments to construct a lower bound of the objective function involved in DMED. As $d$ goes to infinity, the lower bound converges to the objective function and DMED-M converges to DMED while the computational complexity increases. However, DMED-M has no explicit form when $d>4$ and there is no guarantee on the regret gap to the optimality for any finite $d$. Unlike DEMD-M, our UCBoost algorithms can provide guarantees on the complexity and regret performance for arbitrary $\epsilon$, which offers a controlled tradeoff between complexity and optimality.
}

\citeauthor{agarwal2017corralling}~\shortcite{agarwal2017corralling} proposed a boosting technique to obtain a strong bandit algorithm from the existing algorithms, that is adaptive to the environment. However, our boosting technique is specifically designed for stochastic setting and hence allows us to obtain near-optimal algorithms that have better regret gurantees than those obtained using the boosting technique by \citeauthor{agarwal2017corralling}~\shortcite{agarwal2017corralling}.

\section{Problem Formulation}\label{sec:probform}
We consider a stochastic bandit problem with finitely many arms indexed by $a\in\mathcal{K}\triangleq\{1,\ldots,K\}$, where $K$ is a finite positive integer. Each arm $a$ is associated with an unknown probability distribution $v_a$ over the bounded support\footnote{If the supports are bounded in another interval, they can be rescaled to [0,1].} $\Theta=[0,1]$. At each time step $t=1,2,\ldots,$ the agent chooses an action $A_t$ according to past observations (possibly using some independent randomization) and receives a reward $X_{A_t,N_{A_t}(t)}$ independently drawn from the distribution $v_{A_t}$, where $N_a(t)\triangleq\sum_{s=1}^t\mathbbm{1}\{A_s=a\}$ denotes the number of times that arm $a$ was chosen up to time $t$. Note that the agent can only observe the reward $X_{A_t,N_{A_t}(t)}$ at time $t$. Let $\bar{X}_a(t)$ be the empirical mean of arm $a$ based on the observations up to time $t$.

For each arm $a$, we denote by $\mu_a$ the expectation of its associated probability distribution $v_a$. Let $a^*$ be any optimal arm, that is
\begin{equation}
a^*\in\arg\max_{a\in\mathcal{K}}\mu_a.
\end{equation}
We write $\mu^*$ as a shorthand notation for the largest expectation $\mu_{a^*}$ and denote the gap of the expected reward of arm $a$ to $\mu^*$ as $\Delta_a=\mu^*-\mu_a$. The performance of a policy $\pi$ is evaluated through the standard notion of expected regret, which is defined at time horizon $T$ as
\begin{align}
R^\pi(T)&\triangleq \mathbb{E}\left[T\mu^*-\sum_{t=1}^TX_{A_t,N_{A_t}(t)}\right]\\\nonumber
&=\mathbb{E}\left[T\mu^*-\sum_{t=1}^T\mu_{A_t}\right]=\sum_{a\in\mathcal{K}}\Delta_a\mathbb{E}[N_a(T)].
\end{align}
Note that the first equality follows from the tower rule. The goal of the agent is to minimize the expected regret.

\section{Preliminaries}
In this section, we introduce the concept of semi-distance functions and show several related properties. A semi-distance function measures the distance between two expectations of random variables over $\Theta$.
First, we introduce a weak notion of semi-distance function, which is called candidate semi-distance function.
\begin{definition}
\emph{(Candidate semi-distance function)}
A function $d:\Theta\times\Theta\rightarrow \mathbb{R}$ is said to be a candidate semi-distance function if 
\begin{enumerate}
\item $d(p,p)\leq0$, $\forall p\in\Theta;$
\item $d(p,q)\leq d(p,q^\prime)$, $\forall p\leq q\leq q^\prime\in\Theta;$
\item $d(p,q)\geq d(p^\prime,q)$, $\forall p\leq p^\prime\leq q\in\Theta.$
\end{enumerate}
\end{definition}
Clearly, a candidate semi-distance function satisfies the monotone properties\footnote{The monotone properties are equivalent to the triangle inequality in one-dimensional case.} of a distance function. However, it does not need to be non-negative and symmetric. As we show later, such a class of functions plays an important role in the boosting and approximation methods. Moreover, a candidate semi-distance function can be easily modified to a semi-distance function defined as follows.

\begin{definition}
\emph{(Semi-distance function)}
A function $d:\Theta\times\Theta\rightarrow \mathbb{R}$ is said to be a semi-distance function if 
\begin{enumerate}
\item $d(p,q)\geq 0$, $\forall p,q\in\Theta;$
\item $d(p,p)=0$, $\forall p\in\Theta;$
\item $d(p,q)\leq d(p,q^\prime)$, $\forall p\leq q\leq q^\prime\in\Theta;$
\item $d(p,q)\geq d(p^\prime,q)$, $\forall p\leq p^\prime\leq q\in\Theta.$
\end{enumerate}
\end{definition}
A semi-distance function satisfies the non-negative condition, and is stronger than a candidate semi-distance function. The following lemma reveals a simple way to obtain a semi-distance function from a candidate semi-distance function. The proof is provided in Section \ref{app:candsemi}.
\begin{lemma}\label{lem:candsemi}
If $d_1:\Theta\times\Theta\rightarrow \mathbb{R}$ is a candidate semi-distance function and $d_2:\Theta\times\Theta\rightarrow \mathbb{R}$ is a semi-distance function, then $\max(d_1,d_2)$ is a semi-distance function.
\end{lemma}

\begin{remark}\label{remark:0semi-distance}
In particular, $d\equiv0$ is a semi-distance function. So one can easily obtain a semi-distance function from a candidate semi-distance function. 
\end{remark}

As discussed in Remark \ref{remark:0semi-distance}, a semi-distance function may not distinguish two different distributions. So we introduce the following strong notion of semi-distance functions.
\begin{definition}
\emph{(Strong semi-distance function)}
A function $d:\Theta\times\Theta\rightarrow \mathbb{R}$ is said to be a strong semi-distance function if 
\begin{enumerate}
\item $d(p,q)\geq 0$, $\forall p,q\in\Theta;$
\item $d(p,q)=0$, if and only if $p= q\in\Theta;$
\item $d(p,q)\leq d(p,q^\prime)$, $\forall p\leq q\leq q^\prime\in\Theta;$
\item $d(p,q)\geq d(p^\prime,q)$, $\forall p\leq p^\prime\leq q\in\Theta.$
\end{enumerate}
\end{definition}

Similar to Lemma \ref{lem:candsemi}, one can obtain a strong semi-distance function from a candidate semi-distance function as shown in Lemma \ref{lem:candstrong}. The proof of Lemma \ref{lem:candstrong} is provided in Section \ref{app:candstrong}.
\begin{lemma}\label{lem:candstrong}
If $d_1:\Theta\times\Theta\rightarrow \mathbb{R}$ is a candidate semi-distance function and $d_2:\Theta\times\Theta\rightarrow \mathbb{R}$ is a strong semi-distance function, then $\max(d_1,d_2)$ is a strong semi-distance function.
\end{lemma}

A typical strong semi-distance function is the Kullback-Leibler divergence between two Bernoulli distributions,
\begin{equation}
d_{kl}(p,q)=p\log\left(\frac{p}{q}\right)+(1-p)\log\left(\frac{1-p}{1-q}\right).
\end{equation}
In this work, we are interested in semi-distance functions that are dominated by the KL divergence as mentioned above.

\begin{definition}
\emph{(kl-dominated function)}
A function $d:\Theta\times\Theta\rightarrow \mathbb{R}$ is said to be kl-dominated if 
$d(p,q)\leq d_{kl}(p,q),$ $\forall p,q\in\Theta$.
\end{definition}

Consider a set of candidate semi-distance functions. If one can obtain a kl-dominated and strong semi-distance function by taking the maximum, then the set is said to be feasible. A formal definition is presented in Definition \ref{def:feasible}. 

\begin{definition}\label{def:feasible}
\emph{(Feasible set)} A set $D$ of functions from $\Theta\times\Theta$ to $\mathbb{R}$ is said to be feasible if 
\begin{enumerate}
\item $\max\limits_{d\in D}d$ is a strong semi-distance function; 
\item $\max\limits_{d\in D}d$ is kl-dominated.  
\end{enumerate}
\end{definition}

The following proposition shows a sufficient condition for a set to be feasible. 
\begin{proposition}\label{prop:feasible}
A set $D$ of functions from $\Theta\times\Theta$ to $\mathbb{R}$ is feasible if
\begin{enumerate}
\item $\forall d\in D$, $d$ is a candidate semi-distance function; 
\item $\exists d\in D$ such that $d$ is a strong semi-distance function; 
\item $\forall d\in D$, $d$ is kl-dominated.  
\end{enumerate}
\end{proposition}
The proof is provided in Section \ref{app:feasible}. Note that we only need one of the functions to be a strong semi-distance function in order to have a feasible set. This allows us to consider some useful candidate semi-distance functions in our boosting approach.

\section{Boosting}\label{sec:boosting}
We first present a generic form of UCB algorithm, which can generate a class of UCB algorithms that only use the empirical means of the arms. We then provide a boosting technique to obtain a good UCBoost algorithm based on these weak UCB algorithms.
\subsection{The Generic UCB Algorithm}
\begin{algorithm}[t]
\caption{The generic UCB algorithm}
\label{alg:ucb}
\begin{algorithmic}
\REQUIRE semi-distance function $d$
\STATE{Initialization: $t$ {\bfseries from} $1$ {\bfseries to} $K$, play arm $A_t=t$.}
\FOR{$t$ {\bfseries from} $K+1$ {\bfseries to} $T$}
\STATE{Play arm $A_t=\arg\max_{a\in\mathcal{K}}\max\{q\in\Theta:N_a(t-1)d(\bar{X}_a(t-1),q)\leq \log (t)+c\log(\log(t))\}$}
\ENDFOR
\end{algorithmic}
\end{algorithm}

Algorithm \ref{alg:ucb} presents a generic form of UCB algorithm, which only uses the empirical means. The instantiation of the UCB algorithm requires a semi-distance function. Given a semi-distance function $d$, UCB($d$) algorithm finds upper confidence bounds $\{u_a(t)\}_{a\in\mathcal{K}}$ such that the distance $d(\bar{X}_{a}(t-1),u_a(t))$ is at most the exploration bonus ($(\log(t)+c\log(\log(t)))/N_a(t-1)$) for any arm $a$. Note that $c$ is a constant to be determined. In other words, $u_a(t)$ is the solution of the following optimization problem $P_1(d)$,
\begin{align}
P_1(d): \max_{q\in\Theta}~~& q\\
s.t. ~~&d(p,q)\leq \delta,
\end{align}
where $p\in\Theta$ is the empirical mean and $\delta>0$ is the exploration bonus. The computational complexity of the UCB($d$) algorithm depends on the complexity of solving the problem $P_1(d)$. The following result shows that the regret upper bound of the UCB($d$) algorithm depends on the property of the semi-distance function $d$. The detailed proof is presented in Section \ref{app:UCB}.

\begin{theorem}\label{theorem:UCB}
If $d:\Theta\times\Theta\rightarrow \mathbb{R}$ is a strong semi-distance function and is also kl-dominated, then the regret of the UCB($d$) algorithm (generated by plugging $d$ into the generic UCB algorithm) when $c=3$ satisfies:
\begin{equation}
\limsup\limits_{T\to\infty} \frac{\mathbb{E}[R^{\emph{UCB}(d)}(T)]}{\log T}\leq \sum_{a:\mu_a<\mu^*}\frac{\Delta_a}{d(\mu_a,\mu^*)}.
\end{equation}
\end{theorem}
Theorem \ref{theorem:UCB} is a generalization of the regret gurantee of kl-UCB proposed by \citeauthor{garivier2011kl}~\shortcite{garivier2011kl}, which is recovered by UCB($d_{kl}$). Recall that $d_{kl}$ is the KL divergence between two Bernoulli distributions. Note that Theorem \ref{theorem:UCB} holds for general distributions over the support $\Theta$. If the reward distributions are Bernoulli, the kl-UCB algorithm is asymptotically optimal in the sense that the regret of kl-UCB matches the lower bound provided by \citeauthor{lai1985asymptotically}~\shortcite{lai1985asymptotically}:
\begin{equation}
\liminf\limits_{T\to\infty} \frac{\mathbb{E}[R^\pi(T)]}{\log T}\geq \sum_{a:\mu_a<\mu^*}\frac{\Delta_a}{d_{kl}(\mu_a,\mu^*)}.
\end{equation}
However, there is no closed-form solution to the problem $P_1(d_{kl})$. Practical implementation of kl-UCB needs to solve the problem $P_1(d_{kl})$ via numerical methods with high accuracy, which means that the computational complexity is non-trivial.

In addition to the KL divergence function $d_{kl}$, we can find other kl-dominated and strong semi-distance functions such that the complexity of solving $P_1(d)$ is $O(1)$. Then we can obtain some low-complexity UCB algorithms with possibly weak regret performance. For example, consider the $l_2$ distance function,
\begin{equation}
d_{sq}(p,q)=2(p-q)^2.
\end{equation}
It is clear that $d_{sq}$ is a strong semi-distance function. By Pinsker's inequality, $d_{sq}$ is also kl-dominated. Note that UCB($d_{sq}$) recovers the traditional UCB1 algorithm \cite{auer2002finite}, which has been pointed out in \citeauthor{garivier2011kl}~\shortcite{garivier2011kl}.

Now, we introduce two alternative functions to the function $d_{sq}$: biquadratic distance function and Hellinger distance function.  The biquadratic distance function is
\begin{equation}
d_{bq}(p,q)=2(p-q)^2+\frac{4}{9}(p-q)^4.
\end{equation}
The Hellinger distance function\footnote{Actually, $d_h$ is 2 times the square of the Hellinger distance.} is
\begin{equation}
d_{h}(p,q)=\left(\sqrt{p}-\sqrt{q}\right)^2+\left(\sqrt{1-p}-\sqrt{1-q}\right)^2.
\end{equation}
As shown in Lemma \ref{lem:biquadratic} and Lemma \ref{lem:hellinger}, they are kl-dominated strong semi-distance functions and the solutions of the corresponding $P_1(d)$ have closed forms.
\begin{lemma}\label{lem:biquadratic}
The biquadratic distance function $d_{bq}$ is a kl-dominated and strong semi-distance function. The solution of $P_1(d_{bq})$ is
\begin{equation}
q^*=\min\left\{1,p+\sqrt{-\frac{9}{4}+\sqrt{\frac{81}{16}+\frac{9}{4}\delta}}\right\}.
\end{equation}
\end{lemma}
\begin{lemma}\label{lem:hellinger}
The Hellinger distance function $d_{h}$ is a kl-dominated and strong semi-distance function. The solution of $P_1(d_{h})$ is $q^*=$
\begin{align}\nonumber
\left(\left(1-\frac{\delta}{2}\right)\sqrt{p}+\sqrt{(1-p)\left(\delta-\frac{\delta^2}{4}\right)}\right)^{2\times\mathbbm{1}\{\delta<2-2\sqrt{p}\}},
\end{align}
where $\mathbbm{1}\{\cdot\}$ is the indicator function.
\end{lemma}
The proofs of Lemma \ref{lem:biquadratic} and Lemma \ref{lem:hellinger} are presented in Section \ref{app:biquadratic} and \ref{app:hellinger}. The following result follows from Theorem \ref{theorem:UCB} and Lemma \ref{lem:biquadratic}. Note that UCB($d_{bq}$) enjoys the same complexity of UCB1 and better regret guarantee than UCB1.
\begin{corollary}\label{cor:bq}
If $c=3$, then the regret of UCB($d_{bq}$) satisfies
\begin{equation}
\limsup\limits_{T\to\infty} \frac{\mathbb{E}[R^{\emph{UCB}(d_{bq})}(T)]}{\log T}\leq \sum_{a:\mu_a<\mu^*}\frac{\Delta_a}{d_{bq}(\mu_a,\mu^*)}.
\end{equation}
\end{corollary}
The following result follows from Theorem \ref{theorem:UCB} and Lemma \ref{lem:hellinger}. Note that UCB($d_{h}$) enjoys the same complexity of UCB1. In terms of regret guarantees, no one dominates the other in all cases.
\begin{corollary}\label{cor:h}
If $c=3$, then the regret of UCB($d_{h}$) satisfies
\begin{equation}
\limsup\limits_{T\to\infty} \frac{\mathbb{E}[R^{\emph{UCB}(d_{h})}(T)]}{\log T}\leq \sum_{a:\mu_a<\mu^*}\frac{\Delta_a}{d_{h}(\mu_a,\mu^*)}.
\end{equation}
\end{corollary}

\subsection{The UCBoost Algorithm}
\begin{algorithm}[t]
\caption{UCBoost}
\label{alg:ucboost}
\begin{algorithmic}
\REQUIRE candidate semi-distance function set $D$
\STATE{Initialization: $t$ {\bfseries from} $1$ {\bfseries to} $K$, play arm $A_t=t$.}
\FOR{$t$ {\bfseries from} $K+1$ {\bfseries to} $T$}
\STATE{Play arm $A_t=\arg\max_{a\in\mathcal{K}}\min_{d\in D} \max\{q\in\Theta:N_a(t-1)d(\bar{X}_a(t-1),q)\leq \log (t)+c\log(\log(t))\}$}
\ENDFOR
\end{algorithmic}
\end{algorithm}
The generic UCB algorithm provides a way of generating UCB algorithms from semi-distance functions. Among the class of semi-distance functions, some have closed-form solutions of the corresponding problems $P_1(d)$. Thus, the corresponding algorithm UCB($d$) enjoys $O(1)$ computational complexity for each arm in each round. However, these UCB($d$) algorithms are weak in the sense that the regret guarantees of these UCB($d$) algorithms are worse than that of kl-UCB. Moreover, the decision maker does not know which weak UCB($d$) is better when the information $\{\mu_a\}_{a\in\mathcal{K}}$ is unknown. A natural question is: \emph{is there a boosting technique that one can use to obtain a stronger UCB algorithm from these weak UCB algorithms?} The following regret result of Algorithm \ref{alg:ucboost} offers a positive answer.
\begin{theorem}\label{theorem:UCBoost}
If $D$ is a feasible set of candidate semi-distance functions, then the regret of UCBoost($D$) when $c=3$ satisfies:
\begin{equation}\nonumber
\limsup\limits_{T\to\infty} \frac{\mathbb{E}[R^{\emph{UCBoost}(D)}(T)]}{\log T}\leq \sum_{a:\mu_a<\mu^*}\frac{\Delta_a}{\max\limits_{d\in D} d(\mu_a,\mu^*)}.
\end{equation}
\end{theorem}
\begin{proof}
(sketch) We first show that the upper confidence bound of UCBoost($D$) is equivalent to that of UCB($\max_{d\in D}d$). Then the result follows by Theorem~\ref{theorem:UCB}. See detailed proof in Section~\ref{app:UCBoost}.
\end{proof}
The UCBoost algorithm works as the following. Given a feasible set $D$ of candidate semi-distance functions, UCBoost($D$) algorithm queries the upper confidence bound of each weak UCB($d$) once and takes the minimum as the upper confidence bound. Suppose that for any $d\in D$, UCB($d$) enjoys $O(1)$ computational complexity for each arm in each round. Then, UCBoost($D$) enjoys $O(|D|)$ computational complexity for each arm in each round, where $|D|$ is the cardinality of set $D$. Theorem \ref{theorem:UCBoost} shows that UCBoost($D$) has a regret guarantee that is no worse than any UCB($d$) such that $d\in D$. Hence, the UCBoost algorithm can obtain a stronger UCB algorithm from some weak UCB algorithms. Moreover, the following remark shows that the ensemble does not deteriorate the regret performance. 
\begin{remark}
If $D_1$ and $D_2$ are feasible sets, and $D_1\subset D_2$, then the regret guarantee of UCBoost($D_2$) is no worse than that of UCBoost($D_1$).
\end{remark}
By Theorem \ref{theorem:UCBoost}, UCBoost($\{d_{bq},d_h\}$) enjoys the same complexity as UCB1, UCB($d_{bq}$) and UCB($d_h$), and has a no worse regret guarantee. However, the gap between the regret guarantee of UCBoost($\{d_{bq},d_h\}$) and that of kl-UCB may still be large since $d_{bq}$ and $d_h$ are bounded while $d_{kl}$ is unbounded. To address this problem, we are ready to introduce a candidate semi-distance function that is kl-dominated and unbounded. The candidate semi-distance function is a lower bound of the KL divergence function $d_{kl}$,
\begin{equation}
d_{lb}(p,q)=p\log(p)+(1-p)\log\left(\frac{1-p}{1-q}\right).
\end{equation}
\begin{lemma}\label{lem:lb}
The function $d_{lb}$ is a kl-dominated and candidate semi-distance function. The solution of $P_1(d_{lb})$ is 
\begin{equation}
q^*=1-(1-p)\exp\left(\frac{p\log(p)-\delta}{1-p}\right).
\end{equation}
\end{lemma}
The proof of Lemma \ref{lem:lb} is presented in Section \ref{app:lb}. By Lemma \ref{lem:biquadratic}-\ref{lem:lb}, Proposition \ref{prop:feasible} and Theorem \ref{theorem:UCBoost}, we have the following result.
\begin{corollary}\label{cor:ucboost}
If $D=\{d_{bq},d_{h},d_{lb}\}$, then the regret of UCBoost($D$) when $c=3$ satisfies:
\begin{equation}\nonumber
\limsup\limits_{T\to\infty} \frac{\mathbb{E}[R^{\emph{UCBoost}(D)}(T)]}{\log T}\leq \sum_{a:\mu_a<\mu^*}\frac{\Delta_a}{\max\limits_{d\in D} d(\mu_a,\mu^*)}.
\end{equation}
Note that $d_{kl}(\mu_a,\mu^*)-1/e\leq\max\limits_{d\in D} d(\mu_a,\mu^*)\leq d_{kl}(\mu_a,\mu^*)$ for any $a\in\mathcal{K}$ such that $\mu_a<\mu^*$. Thus, we have that 
\begin{equation}\nonumber
\limsup\limits_{T\to\infty} \frac{\mathbb{E}[R^{\emph{UCBoost}(D)}(T)]}{\log T}\leq \sum_{a:\mu_a<\mu^*}\frac{\Delta_a}{d_{kl}(\mu_a,\mu^*)-1/e}.
\end{equation}
\end{corollary}
Although $d_{lb}$ is not a strong semi-distance function, the set $D=\{d_{bq},d_{h},d_{lb}\}$ is still feasible by Proposition \ref{prop:feasible}. The advantage of introducing $d_{lb}$ is that its tightness to $d_{kl}$ improves the regret guarantee of the algorithm. To be specific, the gap between $d_{lb}(\mu_a,\mu^*)$ and $d_{kl}(\mu_a,\mu^*)$ is $\mu_a\log(1/\mu^*)$, which is uniformly bounded by $1/e$ since $\mu_a<\mu^*$. Note that $e$ is the natural number. Hence, UCBoost($\{d_{bq},d_{h},d_{lb}\}$) achieves near-optimal regret performance with low complexity.

Besides the candidate semi-distance function $d_{lb}$, one can find other candidate semi-distance functions and plug them into the set $D$. For example, a shifted tangent line function of $d_{kl}$,
\begin{equation}\nonumber
d_{t}(p,q)=\frac{2q}{p+1}+p\log\left(\frac{p}{p+1}\right)+\log\left(\frac{2}{e(1+p)}\right).
\end{equation}
\begin{lemma}\label{lem:tangent}
The function $d_{t}$ is a kl-dominated and candidate semi-distance function. The solution of $P_1(d_{t})$ is $q^*=$
\begin{equation}\nonumber
\min\left\{1,\frac{p+1}{2}\left(\delta-p\log\left(\frac{p}{p+1}\right)-\log\left(\frac{2}{e(1+p)}\right)\right)\right\}.
\end{equation}
\end{lemma}
The proof is presented in Section \ref{app:tangent}.

\section{The UCBoost($\epsilon$) Algorithm}
In this section, we show an approximation of the KL divergence function $d_{kl}$. Then we design a UCBoost algorithm based on the approximation, which enjoys low complexity and regret guarantee that is arbitrarily close to that of kl-UCB.

Recall that $p\in\Theta$ and $\delta>0$ are the inputs of the problem $P_1(d_{kl})$. Given any approximation error $\epsilon>0$, let $\eta=\frac{\epsilon}{1+\epsilon}$ and $q_k=1-(1-\eta)^k\in\Theta$ for any ${k\geq0}$. Then there exits $\tau_1(p)=\left\lceil\frac{\log(1-p)}{\log(1-\eta)}\right\rceil$ such that $p\leq q_k$ if and only if $k\geq \tau_1(p)$. There exists $\tau_2(p)=\left\lceil\frac{\log(1-\exp(-\epsilon/p))}{\log(1-\eta)}\right\rceil$ such that $q_k\geq \exp(-\epsilon/p)$ if and only if $k\geq\tau_2(p)$. For each $\tau_1(p)\leq k\leq \tau_2(p)$, we construct a step function,
\begin{equation}
d_s^k(p,q)=d_{kl}(p,q_k)\mathbbm{1}\{q>q_k\}.
\end{equation}
The following result shows that the step function $d_s^k(p,q)$ is a kl-dominated and semi-distance function. The proof is presented in Section \ref{app:step}.
\begin{lemma}\label{lem:step}
For each $k\geq \tau_1(p)$, the step function $d_s^k(p,q)$ is a kl-dominated and semi-distance function. The solution of $P_1(d_s^k)$ is 
\begin{equation}
q^*=q_k^{\mathbbm{1}\{\delta<d_{kl}(p,q_k)\}}.
\end{equation}
\end{lemma}

Let $D(p)=\{d_{sq},d_{lb},d_s^{\tau_1(p)},d_s^{\tau_1(p)+1},\ldots,d_s^{\tau_2(p)}\}$. Then the following result shows that the envelope $\max\limits_{d\in D(p)}d$ is an $\epsilon$-approximation of the function $d_{kl}$ on the interval $[p,1]$. The proof is presented in Section \ref{app:stepapprox}.

\begin{proposition}\label{prop:stepapprox}
Given $p\in\Theta$ and $\epsilon >0$. Let $D(p)=\{d_{sq},d_{lb},d_s^{\tau_1(p)},d_s^{\tau_1(p)+1},\ldots,d_s^{\tau_2(p)}\}$.
For any $q\in[p,1]$, we have that
\begin{equation}
0\leq d_{kl}(p,q)-\max\limits_{d\in D(p)}d(p,q)\leq \epsilon.
\end{equation}
\end{proposition} 
Lemma~\ref{lem:step} and Proposition~\ref{prop:stepapprox} allow us to bound the regret of the UCBoost algorithm based on the approximation, which is shown in the following result.

\begin{theorem}\label{theorem:thmstepapprox}
Given any $\epsilon\in(0,1)$, let $D=\{d_{sq},d_{lb}\}\cup\{d_s^k:k\geq 0\}$. The regret of UCBoost($D$) with $c=3$  that restricts $D$ to $D(p)$ for each arm with empirical mean $p$, satisfies  
\begin{equation}
\limsup\limits_{T\to\infty} \frac{\mathbb{E}[R^{\emph{UCBoost}(D)}(T)]}{\log T}\leq \sum_{a:\mu_a<\mu^*}\frac{\Delta_a} {d_{kl}(\mu_a,\mu^*)-\epsilon}.
\end{equation}
The computational complexity for each arm per round is $O(\log(\frac{1}{\epsilon}))$.
\end{theorem}
The proof of Theorem \ref{theorem:thmstepapprox} is presented in Section \ref{app:thmstepapprox}. We denote the algorithm described in Theorem \ref{theorem:thmstepapprox} as UCBoost($\epsilon$) for shorthand. The UCBoost($\epsilon$) algorithm offers an efficient way to trade regret performance with computational complexity.

\begin{remark}
The practical implementation of kl-UCB needs numerical methods for searching the $q^*$ of $P_1(d_{kl})$ with some sufficiently small error $\epsilon$. For example, the bisection search can find a solution $q^\prime$ such that $|q\prime-q^*|\leq \epsilon$ with $O(\log(\frac{1}{\epsilon}))$ iterations. However, there is no regret guarantee of the implemented kl-UCB when $\epsilon$ is arbitrary. Our UCBoost($\epsilon$) algorithm fills this gap and bridges computational complexity to regret performance. Moreover, the empirical performance of the implemented kl-UCB when $\epsilon$ is relatively large, becomes unreliable. This is because the gap $|d_{kl}(p,q^*)-d_{kl}(p,q^\prime)|$ is unbounded even though $|q\prime-q^*|$ is bounded. On the contrary, our approximation method guarantees bounded KL divergence gap, thus allowing reliable regret performance.
\end{remark}

\section{Numerical Results}
\begin{figure*}[t]
\centering
\begin{subfigure}[b]{0.3\textwidth}
    \includegraphics[width=\textwidth]{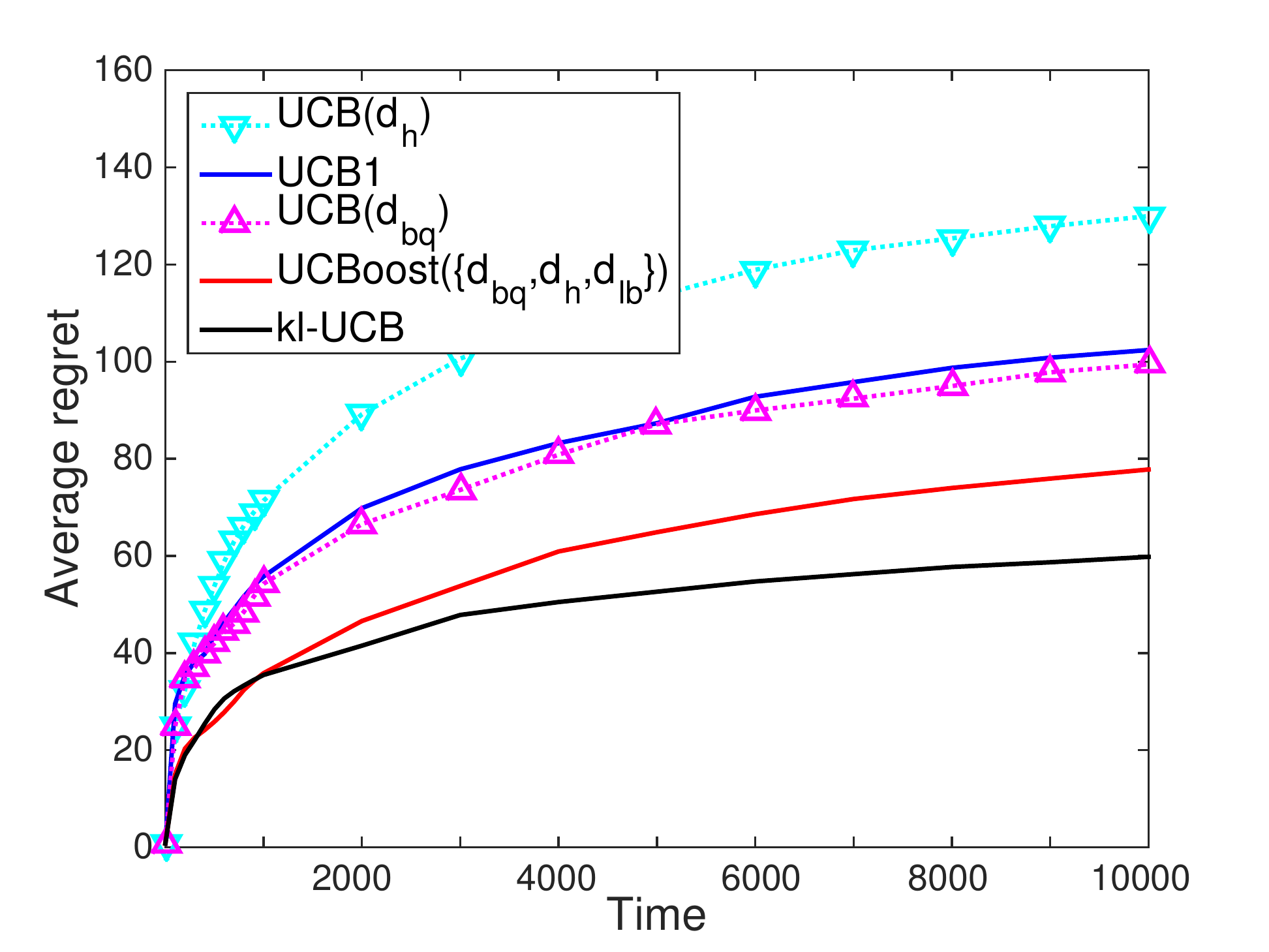}
\end{subfigure}
\begin{subfigure}[b]{0.3\textwidth}
    \includegraphics[width=\textwidth]{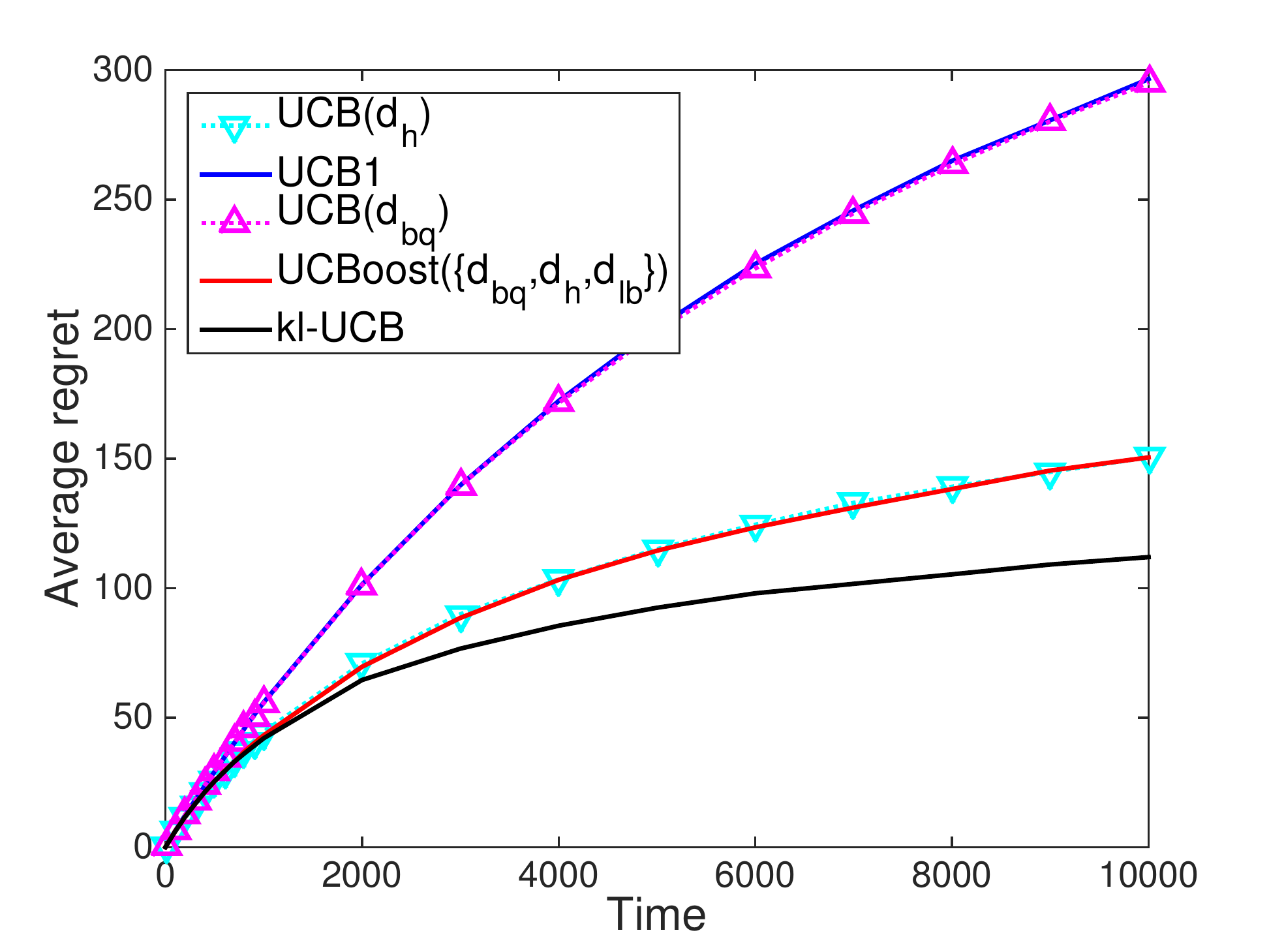}
\end{subfigure}
\begin{subfigure}[b]{0.3\textwidth}
    \includegraphics[width=\textwidth]{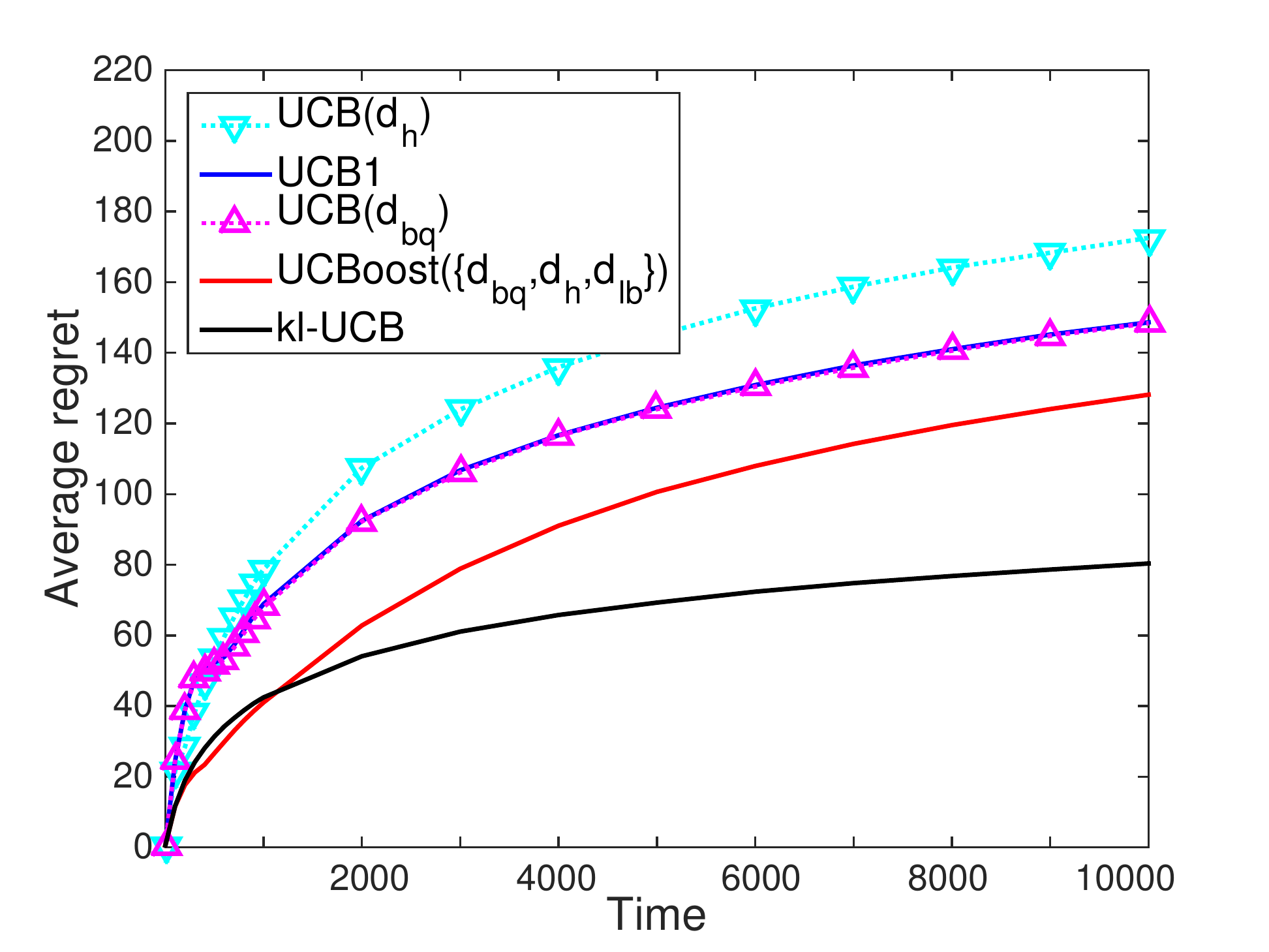}
\end{subfigure}

\begin{subfigure}[b]{0.3\textwidth}
    \includegraphics[width=\textwidth]{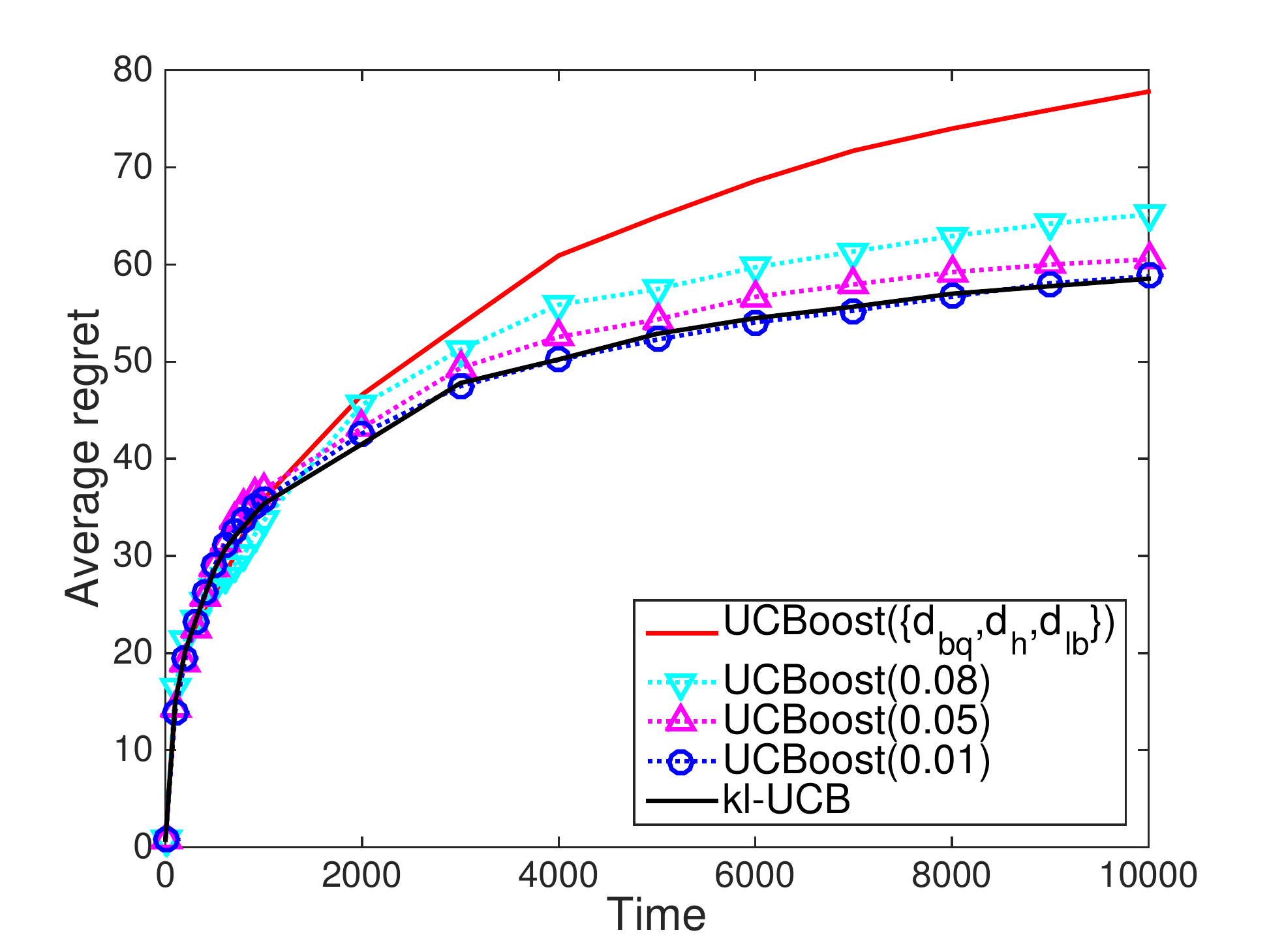}
     \caption{Bernoulli scenario 1}
     \label{fig:epsilonucboost}
\end{subfigure}
\begin{subfigure}[b]{0.3\textwidth}
    \includegraphics[width=\textwidth]{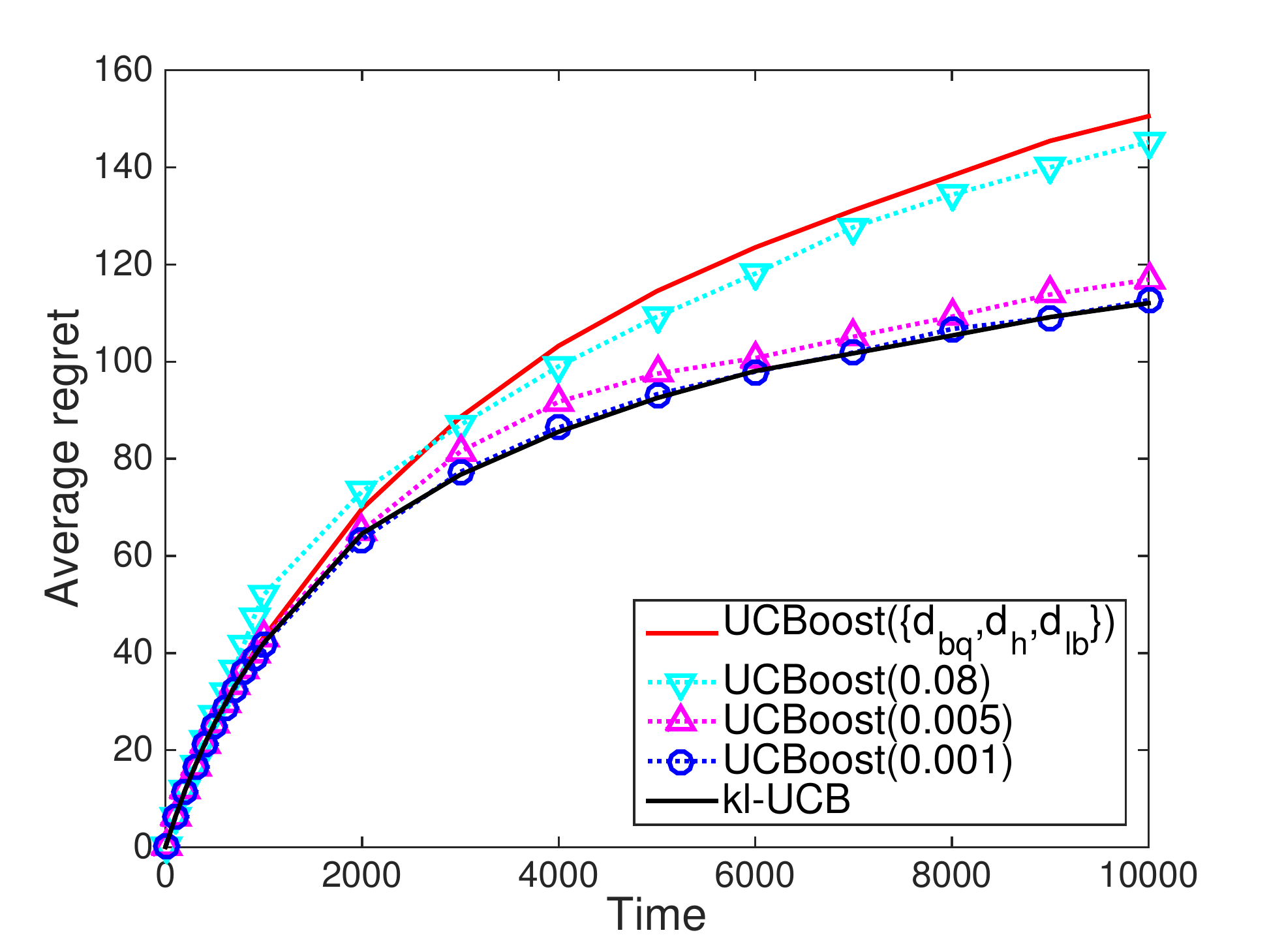}
     \caption{Bernoulli scenario 2}
     \label{fig:lowepsilonucboost}
\end{subfigure}
\begin{subfigure}[b]{0.3\textwidth}
    \includegraphics[width=\textwidth]{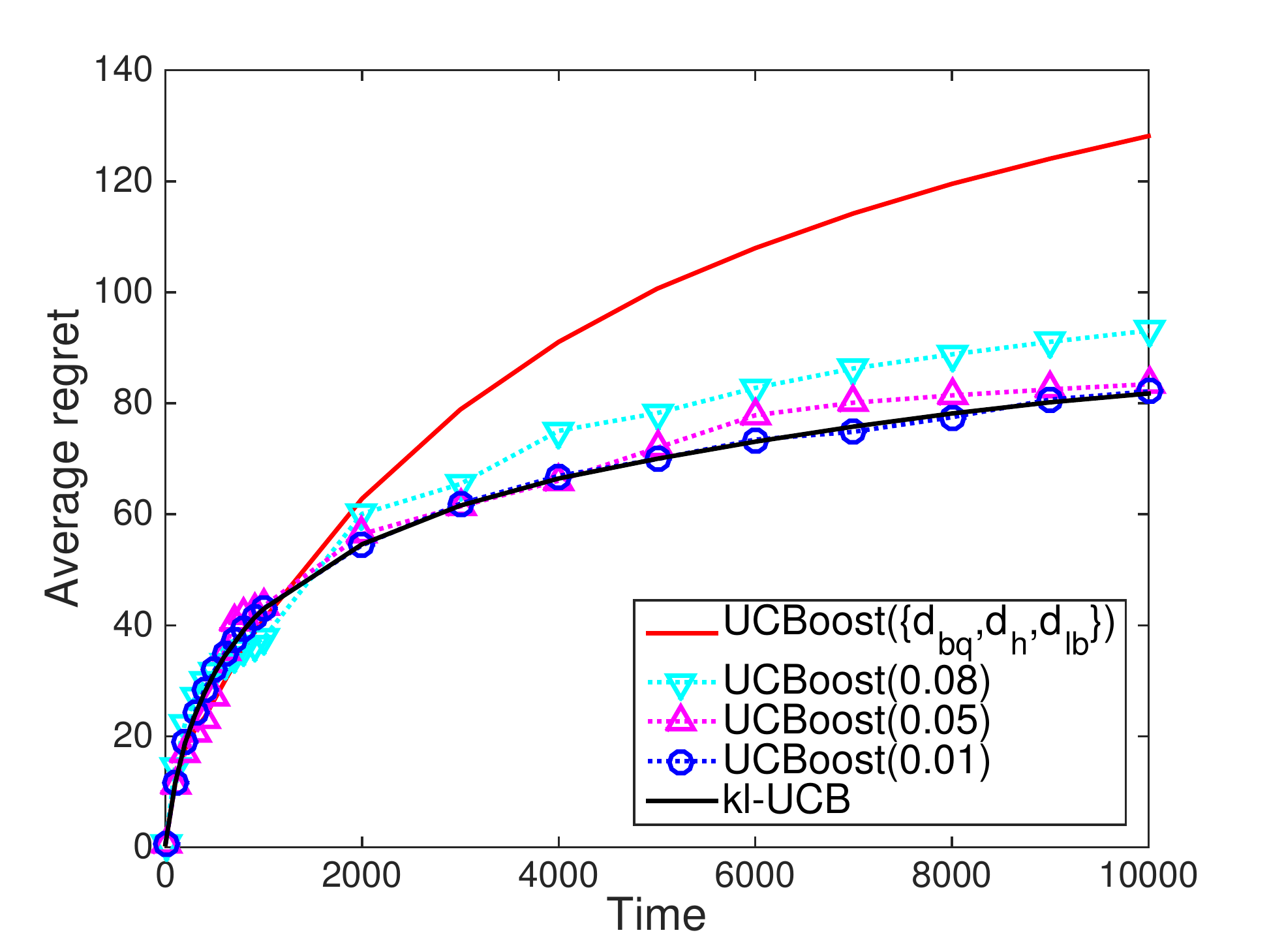}
     \caption{Beta scenario}
     \label{fig:betaepsilonucboost}
\end{subfigure}
\caption{Regret of the various algorithms as a function of time in three scenarios.}
\label{fig:numerical}
\end{figure*}

\begin{table*}[t]
\centering
\caption{Average computational time for each arm per round of various algorithms.}
\label{table:complexity}
\begin{center}
\begin{tabular}{|c|c|c|c|c|c|c|}
\hline
\multirow{2}{*}{Scenario} & \multirow{2}{*}{kl-UCB} & UCBoost($\epsilon$) &UCBoost($\epsilon$)&UCBoost($\epsilon$) &\multirow{2}{*}{UCBoost($\{d_{bq},d_h,d_{lb}\}$)} &\multirow{2}{*}{UCB1}\\
 &  & $\epsilon=0.01 (0.001)$ & $\epsilon=0.05 (0.005)$ & $\epsilon=0.08$&  & \\
 \hline
 Bernoulli 1 & $933\mu s$ & $7.67\mu s$ &$6.67\mu s$&$5.78\mu s$&$1.67\mu s$&$0.31\mu s$\\
 \hline
 Bernoulli 2 & $986\mu s$ & $8.76\mu s$ &$7.96\mu s$&$6.27\mu s$&$1.60\mu s$&$0.30\mu s$\\
 \hline
 Beta & $907\mu s$ & $8.33\mu s$ &$6.89\mu s$&$5.89\mu s$&$2.01\mu s$&$0.33\mu s$\\
 \hline
\end{tabular}
\end{center}
\end{table*}

The results of the previous sections show that UCBoost offers a framework to trade-off the complexity and regret performance. In this section, we support these results by numerical experiments that compare our algorithms with the baseline algorithms in three scenarios. All the algorithms are run exactly as described in the previous sections. For implementation of kl-UCB, we use the py/maBandits package developed by \citeauthor{mabandit}~\shortcite{mabandit}. Note that we choose $c=0$ in the experiments as suggested by \citeauthor{garivier2011kl}~\shortcite{garivier2011kl}. All the results are obtained from $10,000$ independent runs of the algorithms. As shorthand, we denote our approximation-based UCBoost algorithm as UCBoost($\epsilon$) for any $\epsilon>0$. 
\subsection{Bernoulli Scenario 1}
We first consider the basic scenario with Bernoulli rewards. There are $K=9$ arms with expectations $\mu_i=i/10$ for each arm $i$. The average regret of various algorithms as a function of time is shown in Figure~\ref{fig:epsilonucboost}. 

First, UCB($d_{bq}$) performs as expected, though it is slightly better than UCB1. However, UCB($d_h$) performs worse than UCB1 in this scenario. This is reasonable since the regret guarantee of UCB($d_h$) under this scenario is worse than that of UCB1. 

Second, the performance of UCBoost($\{d_{bq},d_h,d_{lb}\}$) is between that of UCB1 and kl-UCB. UCBoost($\{d_{bq},d_h,d_{lb}\}$) outperforms UCB($d_{h}$) and UCB($d_{bq}$) as expected, which demonstrates the power of boosting. The candidate semi-distance function $d_{lb}$ plays an important role in improving the regret performance. 

Third, UCBoost($\epsilon$) algorithm fills the gap between UCBoost($\{d_{bq},d_h,d_{lb}\}$) and kl-UCB with moderate $\epsilon$. As $\epsilon$ decreases, UCBoost($\epsilon$) approaches to kl-UCB, which verifies our result in Theorem \ref{theorem:thmstepapprox}. When $\epsilon=0.01$, UCBoost($\epsilon$) matches the regret of kl-UCB. Note that the numerical method for kl-UCB, such as Newton method and bisection search, usually needs the accuracy to be at least $10^{-5}$. Otherwise, the regret performance of kl-UCB becomes unreliable. Compared to kl-UCB, UCBoost($\epsilon$) can achieve the same regret performance with less complexity by efficiently bounding the KL divergence gap. 
\subsection{Bernoulli Scenario 2}
We consider a more difficult scenario of Bernoulli rewards, where the expectations are very low. This scenario has been considered by \citeauthor{garivier2011kl}~\shortcite{garivier2011kl} to model the situations like online recommendations and online advertising. For example, in Yahoo! Front Page Today experiments \cite{li2010contextual}, the rewards are the click through rates of the news and articles. The rewards are binary and the average click through rates are very low. In this scenario, we consider ten arms, with $\mu_1=\mu_2=\mu_3=0.01$, $\mu_4=\mu_5=\mu_6=0.02$, $\mu_7=\mu_8=\mu_9=0.05$ and $\mu_{10}=0.1$. The average regret of various algorithms as a function of time is shown in Figure \ref{fig:lowepsilonucboost}. 

First, the performance of UCB($d_{bq}$) is the same as UCB1. This is because the term $\Delta_a^4$ vanishes for all suboptimal arms in this scenario. So the improvement of UCB($d_{bq}$) over UCB1 vanishes as well. However, UCB($d_h$) outperforms UCB1 in this scenario. The reason is that the Hellinger distance between $\mu_a$ and $\mu^*$ is much larger than the $l_2$ distance in this scenario. So UCB($d_h$) enjoys better regret performance than UCB1 in this scenario. 

Second, UCBoost($\{d_{bq},d_h,d_{lb}\}$) performs as expected and is between UCB1 and kl-UCB. Although the gap between UCB1 and kl-UCB becomes larger when compared to Bernoulli scenario 1, the gap between UCBoost($\{d_{bq},d_h,d_{lb}\}$) and kl-UCB remains. This verifies our result in Corollary \ref{cor:ucboost} that the gap between the constants in the regret guarantees is bounded by $1/e$. This result also demonstrates the power of boosting in that UCBoost($\{d_{bq},d_h,d_{lb}\}$) performs no worse than UCB($d_h$) and UCB($d_{bq}$) in all cases.

Third, UCBoost($\epsilon$) algorithm fills the gap between UCBoost($\{d_{bq},d_h,d_{lb}\}$) and kl-UCB, which is consistent with the results in Bernoulli scenario 1. The regret of UCBoost($\epsilon$) matches with that of kl-UCB when $\epsilon=0.001$. Compared to the results in Bernoulli scenario 1, we need more accurate approximation for UCBoost when the expectations are lower. However, this accuracy is still moderate compared to the requirements in numerical methods for kl-UCB.
\subsection{Beta Scenario}
Our results in the previous sections hold for any distributions with bounded support. In this scenario, we consider $K=9$ arms with Beta distributions. More precisely, each arm $1\leq i\leq 9$ is associated with Beta($\alpha_i$,$\beta_i$) distribution such that $\alpha_i=i$ and $\beta_i=2$. Note that the expectation of Beta($\alpha_i$,$\beta_i$) is $\alpha_i/(\alpha_i+\beta_i)$. The regret results of various algorithms are shown in Figure \ref{fig:betaepsilonucboost}. The results are consistent with that of Bernoulli scenario 1.

\subsection{Computational time}

We obtain the average running time for each arm per round by measuring the total computational time of $10,000$ independent runs of each algorithms in each scenario. Note that kl-UCB is implemented by the py/maBandits package developed by \citeauthor{mabandit}~\shortcite{mabandit}, which sets accuracy to $10^{-5}$ for the Newton method. The average computational time results are shown in Table \ref{table:complexity}. The average running time of UCBoost($\epsilon$) that matches the regret of kl-UCB is no more than $1\%$ of the time of kl-UCB. 


%
%
%
%
%

\label{sec:simul}
\section{Conclusion}
In this work, we introduce the generic UCB algorithm and provide the regret guarantee for any UCB algorithm generated by a kl-dominated strong semi-distance function. As a by-product, we find two closed-form UCB algorithms, UCB($d_h$) and UCB($d_{bq}$), that are alternatives to the traditional UCB1 algorithm. Then, we propose a boosting framework, UCBoost, to boost any set of generic UCB algorithms. We find a specific finite set $D$, such that UCBoost($D$) enjoys $O(1)$ complexity for each arm per round as well as regret guarantee that is $1/e$-close to the kl-UCB algorithm. Finally, we propose an approximation-based UCBoost algorithm, UCBoost($\epsilon$), that enjoys regret guarantee $\epsilon$-close to that of kl-UCB as well as $O(\log(1/\epsilon))$ complexity for each arm per round. This algorithm bridges the regret guarantee to the computational complexity, thus offering an efficient trade-off between regret performance and complexity for practitioners. By experiments, we show that UCBoost($\epsilon$) can achieve the same regret performance as standard kl-UCB with only $1\%$ computational cost of kl-UCB.

\bibliographystyle{named}
\bibliography{refs}
\appendix
\onecolumn
\section{Proof of Theorem \ref{theorem:UCB}}\label{app:UCB}
\begin{theorem}\label{theorem:concentration}
\emph{(Theorem 10 in \cite{garivier2011kl})} Let $(X_t)_{t\geq1}$ be a sequence of independent random variables bounded in $\Theta$ defined on a probability space $(\Omega,\mathcal{F},\mathbb{P})$ with common expectation $\mu=\mathbb{E}[X_t]$. Let $\mathcal{F}_t$ be an increasing sequence of $\sigma$-fields of $\mathcal{F}$ such that for each $t$, $\sigma(X_1,\ldots,X_t)\subset\mathcal{F}_t$ and for $s>t$, $X_s$ is independent from $\mathcal{F}_t$. Consider a previsible sequence $(\epsilon_t)_{t\geq1}$ of Bernoulli variables (for all $t>0$, $\epsilon_t$ is $\mathcal{F}_{t-1}$-measurable). Let $\delta>0$ and for every $t\in\{1,\ldots,n\}$ let
\begin{align}\nonumber
S(t)&=\sum_{s=1}^t\epsilon_sX_s,~~N(t)=\sum_{s=1}^t\epsilon_s,~~\bar{X}(t)=\frac{S(t)}{N(t)},\\
u(n)&=\max\{q\in\Theta:N(n)d_{kl}(\bar{X}(n),q)\leq\delta\}.
\end{align}
Then
\begin{equation}
\mathbb{P}(u(n)<\mu)\leq e\lceil \delta\log(n)\rceil \exp(-\delta).
\end{equation}
\end{theorem}

\begin{theorem}\label{theorem:dconcentration}
Let $d:\Theta\times\Theta\to\mathbb{R}$ be kl-dominated.
Let $(X_t)_{t\geq1}$ be a sequence of independent random variables bounded in $\Theta$ defined on a probability space $(\Omega,\mathcal{F},\mathbb{P})$ with common expectation $\mu=\mathbb{E}[X_t]$. Let $\mathcal{F}_t$ be an increasing sequence of $\sigma$-fields of $\mathcal{F}$ such that for each $t$, $\sigma(X_1,\ldots,X_t)\subset\mathcal{F}_t$ and for $s>t$, $X_s$ is independent from $\mathcal{F}_t$. Consider a previsible sequence $(\epsilon_t)_{t\geq1}$ of Bernoulli variables (for all $t>0$, $\epsilon_t$ is $\mathcal{F}_{t-1}$-measurable). Let $\delta>0$ and for every $t\in\{1,\ldots,n\}$ let
\begin{align}\nonumber
S(t)&=\sum_{s=1}^t\epsilon_sX_s,~~N(t)=\sum_{s=1}^t\epsilon_s,~~\bar{X}(t)=\frac{S(t)}{N(t)},\\
u(n)&=\max\{q\in\Theta:N(n)d(\bar{X}(n),q)\leq\delta\}. \label{eqn:un}
\end{align}
Then
\begin{equation}
\mathbb{P}(u(n)<\mu)\leq e\lceil \delta\log(n)\rceil \exp(-\delta).
\end{equation}
\end{theorem}
\begin{proof}
Let $q^*(n)=\max\{q\in\Theta:N(n)d_{kl}(\bar{X}(n),q)\leq\delta\}.$ Then we have $N(n)d_{kl}(\bar{X}(n),q^*(n))\leq\delta$. Since $d:\Theta\times\Theta\to\mathbb{R}$ is kl-dominated, we have that
\begin{equation*}
N(n)d(\bar{X}(n),q^*(n))\leq N(n)d_{kl}(\bar{X}(n),q^*(n))\leq\delta.
\end{equation*}
Thus, we have $u(n)\geq q^*(n)$ by the definition in (\ref{eqn:un}).
Theorem \ref{theorem:concentration} implies that
\begin{equation}
\mathbb{P}(q^*(n)<\mu)\leq e\lceil \delta\log(n)\rceil \exp(-\delta).
\end{equation}
Hence, we have that
\begin{equation}
\mathbb{P}(u(n)<\mu)\leq\mathbb{P}(q^*(n)<\mu)\leq e\lceil \delta\log(n)\rceil \exp(-\delta).
\end{equation}
\end{proof}

\begin{theorem}\label{theorem:suboptplay}
Let $d:\Theta\times\Theta\rightarrow \mathbb{R}$ be a strong semi-distance function and kl-dominated. Let $\epsilon>0$, and take $c=3$ in Algorithm \ref{alg:ucb}. For any sub-optimal arm $a$ such that $\mu_a<\mu^*$, the number of times that UCB($d$) algorithm (generated by plugging $d$ into Algorithm \ref{alg:ucb}) chooses arm $a$ is upper-bounded by
\begin{equation}
\mathbb{E}[N_a(T)]\leq \frac{\log (T)}{d(\mu_a,\mu^*)}(1+\epsilon)+C_1\log(\log(T))+\frac{C_2(\epsilon)}{T^{\beta(\epsilon)}},
\end{equation}
where $C_1$ denotes a positive constant and $C_2(\epsilon)$ and $\beta(\epsilon)$ denote positive functions of $\epsilon$. Hence,
\begin{equation}
\limsup_{T\to\infty}\frac{\mathbb{E}[N_a(T)]}{\log(T)}\leq \frac{1}{d(\mu_a,\mu^*)}.
\end{equation}
\end{theorem}
\begin{proof}
Consider $\epsilon>0$ and a sub-optimal arm $a$ such that $\mu_a<\mu^*$. For convenience, we denote the average performance of arm $b$ by $\hat{\mu}_{b,s}=(X_{b,1}+\cdots+X_{b,s})/s$ for any positive integer $s$, so that $\hat{\mu}_{b,N_b(t)}=\bar{X}(t)$. The UCB($d$) algorithm relies on the upper confidence bound $u_b(t)=\max\{q\in\Theta:N_b(t)d(\bar{X}_b(t),q)\leq \log (t)+3\log(\log(t))\}$ for each $\mu_b$.

For any $p,q\in\Theta$, define $d^+(p,q)=d(p,q)\mathbbm{1}_{p<q}$. The expectation of $N_a(T)$ is upper-bounded by the following decomposition:
\begin{align}
\mathbb{E}[N_a(T)]&=\mathbb{E}\left[\sum_{t=1}^T\mathbbm{1}\{A_t=a\}\right]\\
&\leq\mathbb{E}\left[\sum_{t=1}^T\mathbbm{1}\{\mu^*>u_{a^*}(t)\}\right]+\mathbb{E}\left[\sum_{t=1}^T\mathbbm{1}\{A_t=a,\mu^*\leq u_{a^*}(t)\}\right]\\
&\leq\sum_{t=1}^T\mathbb{P}(\mu^*>u_{a^*}(t))+\mathbb{E}\left[\sum_{s=1}^T\mathbbm{1}\{sd^+(\hat{\mu}_{a,s},\mu^*)<\log(T)+3\log(\log(T))\}\right],
\end{align}
where the last inequality follows from Lemma \ref{lem:2ndterm}. The first term is upper-bounded by Theorem \ref{theorem:dconcentration},
\begin{align}
\sum_{t=1}^T\mathbb{P}(u_{a^*}(t)<\mu^*)&\leq \sum_{t=1}^Te\lceil(\log(t)+3\log(\log(t)))\log(t)\rceil\exp(-\log(t)-3\log(\log(t)))\\
\leq C_1^\prime\log(\log(T))
\end{align}
for some positive constant $C_1^\prime$ ($C_1^\prime=7$ is sufficient). Now it remains to bound the second term. We define the following shorthand
$$K_T=\left\lfloor\frac{1+\epsilon}{d(\mu_a,\mu^*)}\left(\log(T)+3\log(\log(T))\right)\right\rfloor$$.
Then, we have that
\begin{align}
\mathbb{E}&\left[\sum_{s=1}^T\mathbbm{1}\{sd^+(\hat{\mu}_{a,s},\mu^*)<\log(T)+3\log(\log(T))\}\right]=\sum_{s=1}^T\mathbb{P}\left(sd^+(\hat{\mu}_{a,s},\mu^*)<\log(T)+3\log(\log(T))\right)\\
&\leq K_T+\sum_{s=K_T+1}^\infty\mathbb{P}\left(sd^+(\hat{\mu}_{a,s},\mu^*)<\log(T)+3\log(\log(T))\right)\\
&\leq K_T+\sum_{s=K_T+1}^\infty\mathbb{P}\left(K_Td^+(\hat{\mu}_{a,s},\mu^*)<\log(T)+3\log(\log(T))\right)\\
&\leq K_T+\sum_{s=K_T+1}^\infty\mathbb{P}\left(d^+(\hat{\mu}_{a,s},\mu^*)<\frac{d(\mu_a,\mu^*)}{1+\epsilon}\right)\\
&\leq \frac{1+\epsilon}{d(\mu_a,\mu^*)}\left(\log(T)+3\log(\log(T))\right)+\frac{C_2(\epsilon)}{T^\beta(\epsilon)}
\end{align}
where the last inequality follows from Lemma \ref{lem:2ndbound}. The result follows.
\end{proof}

\begin{lemma}\label{lem:2ndterm}
If $d:\Theta\times\Theta\rightarrow \mathbb{R}$ is a semi-distance function, then
$$\sum_{t=1}^T\mathbbm{1}\{A_t=a,\mu^*\leq u_{a^*}(t)\}\leq\sum_{s=1}^T\mathbbm{1}\{sd^+(\hat{\mu}_{a,s},\mu^*)<\log(T)+3\log(\log(T))\}.$$
\end{lemma}
\begin{proof}
It is clear that $A_t=a$ and $\mu^*\leq u_{a^*}(t)$ implies that $u_a(t)\geq u_{a^*}(t)\geq\mu^*$. By the definition of $u_a(t)$, we have that $N_a(t)d(\bar{X}_a(t),u_a(t))\leq \log(t)+3\log(\log(t))$. Since $d$ is a semi-distance function, we have that $d(\bar{X}_a(t),\mu^*)\leq d(\bar{X}_a(t),u_a(t))$ if $\bar{X}_a(t)\leq\mu^*\leq u_a(t)$. Hence, $A_t=a$ and $\mu^*\leq u_{a^*}(t)$ implies that
\begin{equation}
d^+(\bar{X}_a(t),\mu^*)\leq d(\bar{X}_a(t),u_a(t))\leq\frac{\log(t)+3\log(\log(t))}{N_a(t)}.
\end{equation}
Note that $d^+(p,q)=d(p,q)\mathbbm{1}_{p<q}$. Thus, we have 
\begin{align}
\sum_{t=1}^T\mathbbm{1}\{A_t=a,\mu^*\leq u_{a^*}(t)\}&\leq\sum_{t=1}^T\mathbbm{1}\{A_t=a,N_a(t)d^+(\bar{X}_a(t),\mu^*)\leq{\log(t)+3\log(\log(t))}\}\\
&=\sum_{t=1}^T\sum_{s=1}^t\mathbbm{1}\{N_a(t)=s,A_t=a,sd^+(\hat{\mu}_{a,s},\mu^*)\leq{\log(t)+3\log(\log(t))}\}\\
&=\sum_{t=1}^T\sum_{s=1}^t\mathbbm{1}\{N_a(t)=s,A_t=a\}\mathbbm{1}\{sd^+(\hat{\mu}_{a,s},\mu^*)\leq{\log(t)+3\log(\log(t))}\}\\
&\leq\sum_{t=1}^T\sum_{s=1}^t\mathbbm{1}\{N_a(t)=s,A_t=a\}\mathbbm{1}\{sd^+(\hat{\mu}_{a,s},\mu^*)\leq{\log(T)+3\log(\log(T))}\}\\
&=\sum_{s=1}^T\sum_{t=s}^T\mathbbm{1}\{N_a(t)=s,A_t=a\}\mathbbm{1}\{sd^+(\hat{\mu}_{a,s},\mu^*)\leq{\log(T)+3\log(\log(T))}\}\\
&=\sum_{s=1}^T\mathbbm{1}\{sd^+(\hat{\mu}_{a,s},\mu^*)\leq{\log(T)+3\log(\log(T))}\}\sum_{t=s}^T\mathbbm{1}\{N_a(t)=s,A_t=a\}\\
&\leq\sum_{s=1}^T\mathbbm{1}\{sd^+(\hat{\mu}_{a,s},\mu^*)\leq{\log(T)+3\log(\log(T))}\},
\end{align}
where the last inequality follows from $\sum_{t=s}^T\mathbbm{1}\{N_a(t)=s,A_t=a\}\leq 1$ for any $s$.
\end{proof}

\begin{lemma}\label{lem:2ndbound}
Let $d:\Theta\times\Theta\rightarrow \mathbb{R}$ be a strong semi-distance function.
Given $\epsilon>0$, there exist $C_2(\epsilon)>0$ and $\beta(\epsilon)>0$ such that
$$\sum_{s=K_T+1}^\infty\mathbb{P}\left(d^+(\hat{\mu}_{a,s},\mu^*)<\frac{d(\mu_a,\mu^*)}{1+\epsilon}\right)\leq\frac{C_2(\epsilon)}{T^\beta(\epsilon)}$$.
\end{lemma}
\begin{proof}
Observe that $d^+(\hat{\mu}_{a,s},\mu^*)<{d(\mu_a,\mu^*)}/{(1+\epsilon)}$ implies that $\hat{\mu}_{a,s}>r(\epsilon)$, where $r(\epsilon)\in(\mu_a,\mu^*)$ such that $d(r(\epsilon),\mu^*)=d(\mu_a,\mu^*)/(1+\epsilon)$. Note that $r(\epsilon)$ exists because $d$ is a strong semi-distance function. Thus, we have
\begin{equation}
\mathbb{P}\left(d^+(\hat{\mu}_{a,s},\mu^*)<\frac{d(\mu_a,\mu^*)}{1+\epsilon}\right)\leq \mathbb{P}(\hat{\mu}_{a,s}>r(\epsilon))\leq\exp(-sd_{kl}(r(\epsilon),\mu_a)).
\end{equation}
Hence,
\begin{equation}
\sum_{s=K_T+1}^\infty\mathbb{P}\left(d^+(\hat{\mu}_{a,s},\mu^*)<\frac{d(\mu_a,\mu^*)}{1+\epsilon}\right)\leq\sum_{s=K_T+1}^\infty\exp(-sd_{kl}(r(\epsilon),\mu_a))\leq\frac{\exp(-d_{kl}(r(\epsilon),\mu_a)K_T)}{1-\exp(-d_{kl}(r(\epsilon),\mu_a))}\leq\frac{C_2(\epsilon)}{T^{\beta(\epsilon)}},
\end{equation}
where $C_2(\epsilon)=(1-\exp(-d_{kl}(r(\epsilon),\mu_a)))^{-1}$ and $\beta(\epsilon)=(1+\epsilon)d_{kl}(r(\epsilon),\mu_a)/d(\mu_a,\mu^*)$.
\end{proof}

\section{Proofs}
\subsection{Proof of Theorem \ref{theorem:UCBoost}}\label{app:UCBoost}
\begin{proof}
Let $u_a^d(t)=\max\{q\in\Theta:N_a(t-1)d(\bar{X}_a(t-1),q)\leq \log (t)+c\log(\log(t))\}$. Then $u_a(t)=\min\limits_{d\in D}u_a^d(t)$ is the upper confidence bound that UCBoost($D$) assigns to arm $a$.
Let $u_a^D(t)=\max\{q\in\Theta:N_a(t-1)\max_{d\in D}d(\bar{X}_a(t-1),q)\leq \log (t)+c\log(\log(t))\}$. We claim that $u_a(t)=u_a^D(t)$. 

First, $u_a(t)\leq u_a^d(t)$ $\forall d\in D$ implies that 
\begin{equation}
N_a(t-1)d(\bar{X}_a(t-1),u_a(t))\leq \log (t)+c\log(\log(t)) , \forall d\in D.
\end{equation}
Thus, we have that
\begin{equation}
N_a(t-1)\max_{d\in D}d(\bar{X}_a(t-1),u_a(t))\leq \log (t)+c\log(\log(t)).
\end{equation}
So we have that $u_a(t)\leq u_a^D(t)$.

Second, we have that
\begin{equation}
N_a(t-1)\max_{d\in D}d(\bar{X}_a(t-1),u_a^D(t))\leq \log (t)+c\log(\log(t)).
\end{equation}
Thus, we have that
\begin{equation}
N_a(t-1)d(\bar{X}_a(t-1),u_a^D(t))\leq \log (t)+c\log(\log(t)) , \forall d\in D.
\end{equation}
So we have that $u_a^D(t)\leq u_a^d(t)$ $\forall d\in D$. Hence, $u_a^D(t)\leq \min_{d\in D}u_a^d(t)=u_a(t)$.

So we show that $u_a(t)=u_a^D(t)$. The result follows from Theorem \ref{theorem:UCB} and Definition \ref{def:feasible}. Note that $\max\limits_{d\in D}d$ is a strong semi-distance function and kl-dominated.
\end{proof}
\subsection{Proof of Lemma \ref{lem:candsemi}}\label{app:candsemi}
\begin{proof}
Suppose that $d_1:\Theta\times\Theta\rightarrow \mathbb{R}$ is a candidate semi-distance function and $d_2:\Theta\times\Theta\rightarrow \mathbb{R}$ is a semi-distance function. It remains to check that $d=\max(d_1,d_2)$ satisfies the definition of semi-distance function.

First, for any $p,q\in\Theta$, we have that
\begin{equation}
d(p,q)=\max\{d_1(p,q),d_2(p,q)\}\geq d_2(p,q)\geq0,
\end{equation}
since $d_2$ is a semi-distance function. 

Second, for any $p\in\Theta$, we have that
\begin{equation}
d(p,p)=\max\{d_1(p,p),d_2(p,p)\}=\max\{d_1(p,q),0\}=0,
\end{equation}
since $d_1(p,p)\leq0$.

Third, for any $p\leq q\leq q^\prime\in\Theta$, we have that
\begin{align}
d_1(p,q)&\leq d_1(p,q^\prime)\\
d_2(p,q)&\leq d_2(p,q^\prime)
\end{align}
Thus, we have that
\begin{equation}
d(p,q)=\max\{d_1(p,q),d_2(p,q)\}\leq\max\{d_1(p,q^\prime),d_2(p,q^\prime)\}=d(p,q^\prime).
\end{equation}
Similarly, we have that
\begin{equation}
d(p,q)\geq d(p^\prime,q), \forall p\leq p^\prime\leq q\in\Theta.
\end{equation}
Hence, $d=\max(d_1,d_2)$ is semi-distance function.
\end{proof}
\subsection{Proof of Lemma \ref{lem:candstrong}}\label{app:candstrong}
\begin{proof}
Suppose that $d_1:\Theta\times\Theta\rightarrow \mathbb{R}$ is a candidate semi-distance function and $d_2:\Theta\times\Theta\rightarrow \mathbb{R}$ is a strong semi-distance function. By Lemma \ref{lem:candsemi}, we have that $d=\max(d_1,d_2)$ is a semi-distance function. Then it remains to check the sufficient and necessary condition. If $p\neq q\in\Theta$, then $d_2(p,q)>0$ implies that $d(p,q)=\max\{d_1(p,q),d_2(p,q)\}\geq d_2(p,q)>0$. Hence, $d=\max(d_1,d_2)$ is a strong semi-distance function. 
\end{proof}
\subsection{Proof of Proposition \ref{prop:feasible}}\label{app:feasible}
\begin{proof}
Suppose that $D=\{d_1,\ldots,d_M\}$ for some positive integer and $D$ satisfies the conditions. It remains to check that $\max\limits_{d\in D} d$ is a strong semi-distance function. Without loss of generality, we assume that $d_1\in D$ is a strong semi-distance function. Let $\hat{d}_k=\max(\hat{d}_{k-1},d_{k})$ for $k\geq2$ and $\hat{d}_1=d_1$. It is clear that $\max\limits_{d\in D} d=\hat{d}_M$. By Lemma \ref{lem:candstrong}, $\hat{d}_k$ is a strong semi-distance function. Hence, $\max\limits_{d\in D} d$ is a strong semi-distance function.
\end{proof}

\subsection{Proof of Lemma \ref{lem:biquadratic}}\label{app:biquadratic}
\begin{proof}
It has been shown by \citeauthor{kullback1967lower}~\shortcite{kullback1967lower} and \citeauthor{kullback1970correction}~\shortcite{kullback1970correction} that $d_{kl}(p,q)\geq d_{bq}(p,q)$ for any $p,q\in\Theta$. It is straightforward to check that $d_{bq}$ is a strong semi-distance function. The solution of $P_1(d_{bq})$ is the root of a biquadratic function.
\end{proof}

\subsection{Proof of Lemma \ref{lem:hellinger}}\label{app:hellinger}
\begin{proof}
By the inequality $1-x\leq-\log x$, we have that
\begin{align}
d_h(p,q)&=(\sqrt{p}-\sqrt{q})^2+\left(\sqrt{1-p}-\sqrt{1-q}\right)^2\\
&=2(p-\sqrt{pq})+2\left(1-p-\sqrt{(1-p)(1-q)}\right)\\
&=2p\left(1-\sqrt{\frac{q}{p}}\right)+2(1-p)\left(1-\sqrt{\frac{1-q}{1-p}}\right)\\
&\leq-2p\log\left(\sqrt{\frac{q}{p}}\right)-2(1-p)\log\left(\sqrt{\frac{1-q}{1-p}}\right)\\
&=p\log\left(\frac{p}{q}\right)+(1-p)\log\left(\frac{1-p}{1-q}\right)\\
&=d_{kl}(p,q).
\end{align}
Thus, $d_h$ is kl-dominated. It is straightforward to check that $d_{bq}$ is a strong semi-distance function. The solution of $P_1(d_{h})$ is the root of a quadratic function.
\end{proof}

\subsection{Proof of Lemma \ref{lem:lb}}\label{app:lb}
\begin{proof}
First, the function $d_{lb}$ is kl-dominated since $d_{kl}(p,q)-d_{lb}(p,q)=-p\log(q)\geq 0$.

Now, we check the conditions of a candidate semi-distance function.
\begin{enumerate}
\item $d_{lb}(p,p)=p\log(p)\leq 0$ holds for any $p\in\Theta$.
\item $\frac{\partial d_{lb}(p,q)}{\partial q}=\frac{1-p}{1-q}\geq 0$. Thus, $d_{lb}(p,q)\leq d_{lb}(p,q^\prime)$, $\forall p\leq q\leq q^\prime\in\Theta$
\item $\frac{\partial d_{lb}(p,q)}{\partial p}=\log(p)-\log\left(\frac{1-p}{1-q}\right)\leq 0$ if $p\leq q \in \Theta$. Thus, $d_{lb}(p,q)\geq d_{lb}(p^\prime,q)$, $\forall p\leq p^\prime\leq q\in\Theta$.
\end{enumerate}

The solution of $P_1(d_{lb})$ is the root of a simple function.
\end{proof}

\subsection{Proof of Lemma \ref{lem:tangent}}\label{app:tangent}
\begin{proof}
Given $p\in\Theta$, $d_t(p,\cdot)+\log(1+p)$ is the tangent line of $d_{kl}(p,\cdot)$ through the point $\left(\frac{1+p}{2},d_{kl}\left(p,\frac{1+p}{2}\right)\right)$. Thus, the function $d_{t}$ is kl-dominated since $d_{kl}$ is convex.

Now, we check the conditions of a candidate semi-distance function.
\begin{enumerate}
\item $d_{t}(p,p)=\frac{2p}{p+1}+p\log\left(\frac{p}{p+1}\right)+\log\left(\frac{2}{e(1+p)}\right)$ is decreasing in $p\in\Theta$. And $\lim_{p\to0}d_{t}(p,p)=\log(2/e)<0$. Hence, we have that $d_{t}(p,p)\leq 0$ for any $p\in\Theta$.
\item $\frac{\partial d_{t}(p,q)}{\partial q}=\frac{2}{1+p}\geq 0$. Thus, $d_{t}(p,q)\leq d_{t}(p,q^\prime)$, $\forall p\leq q\leq q^\prime\in\Theta$
\item $\frac{\partial d_{t}(p,q)}{\partial p}=-\frac{2q}{(1+p)^2}+\log\left(\frac{p}{1+p}\right)\leq 0$ if $p\leq q \in \Theta$. Thus, $d_{t}(p,q)\geq d_{t}(p^\prime,q)$, $\forall p\leq p^\prime\leq q\in\Theta$.
\end{enumerate}

The solution of $P_1(d_{t})$ is the root of a simple function.
\end{proof}

\subsection{Proof of Lemma \ref{lem:step}}\label{app:step}
\begin{proof}
Given $k\geq \tau_1(p)$, we have that $p\leq q_k$. Thus, $d_s^k(p,q)=d_{kl}(p,q_k)\mathbbm{1}\{q>q_k\}\leq d_{kl}(p,q)$ for any $p,q\in\Theta$. So $d_s^k$ is kl-dominated.

Now, we check the conditions of a semi-distance function.
\begin{enumerate}
\item $d_s^k(p,q)=d_{kl}(p,q_k)\mathbbm{1}\{q>q_k\}\geq 0$ since $d_{kl}(p,q_k)\geq 0$.
\item $d_{s}^k(p,p)=d_{kl}(p,q_k)\mathbbm{1}\{p>q_k\}=0$ for any $p\in\Theta$.
\item $\forall p\leq q\leq q^\prime\in\Theta$, $d_s^k(p,q)=d_{kl}(p,q_k)\mathbbm{1}\{q>q_k\}\leq d_{kl}(p,q_k)\mathbbm{1}\{q^\prime>q_k\}=d_s^k(p,q^\prime)$.
\item $\frac{\partial d_{s}^k(p,q)}{\partial p}=\log\left(\frac{p(1-q_k)}{q_k(1-p)}\right)\mathbbm{1}\{q>q_k\}\leq 0$. Thus, $d_s^k(p,q)\geq d_s^k(p^\prime,q)$, $\forall p\leq p^\prime\leq q\in\Theta$.
\end{enumerate}

The solution of $P_1(d_{s}^k)$ is the root of a simple function.
\end{proof}

\subsection{Proof of Proposition \ref{prop:stepapprox}}\label{app:stepapprox}
\begin{proof}
Given $p\in\Theta$ and $\epsilon>0$. For any $q\in[p,1]$, we have that $d_{kl}(p,q)-\max\limits_{d\in D(p)}d(p,q)\geq 0$ since all the functions in $D(p)$ are kl-dominated. 

For any $q\in[\exp(-\epsilon/p),1]$, we have that $d_{kl}(p,q)-d_{lb}(p,q)=-p\log(q)\leq-p\log(\exp(-\epsilon/p))=\epsilon$. Thus, $d_{kl}(p,q)-\max\limits_{d\in D(p)}d(p,q)\leq d_{kl}(p,q)-d_{lb}(p,q)\leq \epsilon$.

For any $q\in[p,\exp(-\epsilon/p)]$, we consider the piece-wise constant function formulated by $\max\limits_{\tau_1(p)\leq k\leq \tau_2(p)}d_s^k$. Let $L_{k}=\frac{\partial d_{kl}(p,q)}{\partial q}|_{q=q_k}$ denote the gradient of $d_{kl}(p,\cdot)$ at the point $q=q_k$. Then we have that $L_k=\frac{1-p}{1-q_k}-\frac{p}{q_k}\leq \frac{1}{1-q_k}$. For any $q\in(q_{\tau_1(p)},\exp(-\epsilon/p)]$, there exists $k$ such that $q_k< q\leq q_{k+1}$ By the convexity of $d_{kl}$, we have that 
\begin{align}
d_{kl}(p,q)-\max\limits_{\tau_1(p)\leq i\leq \tau_2(p)}d_s^i(p,q)&=d_{kl}(p,q)-d^k_s(p,q)\\
&=d_{kl}(p,q)-d_{kl}(p,q_k)\\
&\leq d_{kl}(p,q_{k+1})-d_{kl}(p,q_k)\\
&\leq (q_{k+1}-q_k)L_{k+1}\\
&\leq \frac{q_{k+1}-q_k}{1-q_{k+1}}\\
&=\epsilon.
\end{align}
For any $q\in[p,q_{\tau_1(p)}]$, we have that
\begin{align}
d_{kl}(p,q)-\max\limits_{\tau_1(p)\leq i\leq \tau_2(p)}d_s^i(p,q)&=d_{kl}(p,q)\\
&\leq d_{kl}(p,q_{\tau_1(p)})\\
&\leq (q_{\tau_1(p)}-p)L_{\tau_1(p)}\\
&\leq (q_{\tau_1(p)}-q_{\tau_1(p)-1})L_{\tau_1(p)}\\
&\leq\epsilon.
\end{align}
Thus, for any $q\in[p,\exp(-\epsilon/p)]$ we have that $d_{kl}(p,q)-\max\limits_{d\in D(p)}d(p,q)\leq d_{kl}(p,q)-\max\limits_{\tau_1(p)\leq i\leq \tau_2(p)}d_s^i(p,q)\leq \epsilon$. 
\end{proof}

\subsection{Proof of Theorem \ref{theorem:thmstepapprox}}\label{app:thmstepapprox}
\begin{proof}
The regret result follows by Lemma \ref{lem:lb}, Lemma \ref{lem:step}, Proposition \ref{prop:feasible}, Theorem \ref{theorem:UCBoost} and Proposition \ref{prop:stepapprox}. It remains to show the complexity result. For each given $p$, the complexity is to solve the problem $P_1\left(\max\limits_{d\in D(p)}d\right)$.
Since $\max\limits_{d\in D(p)}d=\max\left\{d_{sq},d_{lb},\max\limits_{\tau_1(p)\leq k\leq\tau_2(p)}d^k_s\right\}$, it is equivalent to find the minimum of $P_1(d_{sq})$, $P_1(d_{lb})$ and $P_1\left(\max\limits_{\tau_1(p)\leq k\leq\tau_2(p)}d^k_s\right)$. By the non-decreasing and piece-wise constant structure of $\max\limits_{\tau_1(p)\leq k\leq\tau_2(p)}d^k_s$, one can use the standard bisection search to find the root. Thus, the computational complexity is $O(\log(\tau_2(p)-\tau_1(p)+1))$.

Recall that $\tau_2(p)=\left\lceil\frac{\log(1-\exp(-\epsilon/p))}{\log(1-\eta)}\right\rceil$. Then, we have that
\begin{align}
\log(\tau_2(p)-\tau_1(p)+1)\leq\log\left(-\frac{\log(1-\exp(-\epsilon/p))}{\log(1+\epsilon)}\right)\leq\log\left(\frac{-2\log(1-\exp(-\epsilon/p))}{\epsilon}\right),
\end{align}
by the fact that $\eta=\frac{\epsilon}{1+\epsilon}$ and $\log(1+\epsilon)\geq \epsilon/2$ for any $\epsilon\in(0,1)$. Now, we claim that $-\log(1-\exp(-\epsilon/p))\leq \frac{p}{\epsilon}$, which is equivalent to $\exp(-p/\epsilon)+\exp(-\epsilon/p)\leq 1$. We define the function $f(x)=\exp(-x)+\exp(-1/x)$. Thus, we have that $f^\prime(x)=-\exp(-x)+\exp(-1/x)/x^2$. Observe that $f^\prime(x)\geq 0$ if and only if $x\geq 1$. Thus, the function $f(x)$ is decreasing on $(0,1)$ and increasing on $(1,\infty)$. Note that $\lim_{x\to 0}f(x)=1=\lim_{x\to\infty}f(x)$. Hence we have that $f(x)\leq 1$ for any $x>0$. So we have that $\exp(-p/\epsilon)+\exp(-\epsilon/p)\leq 1$. 
Hence, we have that the complexity is at most $O(\log(1/\epsilon))$.
\end{proof}

\subsection{Step Function Approximation-Based Bisection Search}
Algorithm~\ref{alg:stepapprox} shows the procedure to solve the problem $P_1\left(\max\limits_{d\in D(p)}d\right)$.
\begin{algorithm}[t]
\caption{Step Function Approximation-Based Bisection Search}
\label{alg:stepapprox}
\begin{algorithmic}
\REQUIRE empirical mean $p$, exploration bonus $\delta$ and approximation error $\epsilon$
\STATE{Initialization: $\tau_1=\left\lceil\frac{-\log(1-p)}{\log(1+\epsilon)}\right\rceil$, $\tau_2=\left\lceil\frac{-\log(1-\exp(-\epsilon/p))}{\log(1+\epsilon)}\right\rceil$ and $k=\left\lfloor \frac{\tau_1+\tau_2}{2}\right\rfloor$}
\IF{$\tau_1\leq\tau_2$}
\IF{$d_{kl}(p,q_{\tau_2})<\delta$}
\STATE{$q=1-(1-p)\exp\left(\frac{p\log(p)-\delta}{1-p}\right)$}
\ELSIF{$d_{kl}(p,q_{\tau_1})\geq\delta$}
\STATE{$q=q_{\tau_1}$}
\ELSE
\WHILE{\TRUE}
\IF{$d_{kl}(p,q_{k-1})<\delta\leq d_{kl}(p,q_k)$}
\STATE{break;}
\ELSIF{$\delta> d_{kl}(p,q_k)$}
\STATE{$k=\left\lfloor\frac{k+\tau_2}{2}\right\rfloor$}
\ELSE
\STATE{$k=\left\lfloor\frac{k+\tau_1}{2}\right\rfloor$}
\ENDIF
\ENDWHILE
\STATE{$q=q_k$}
\ENDIF
\RETURN $\min\left(q,p+\sqrt{\delta/2}\right)$
\ELSE
\RETURN $\min\left(1-(1-p)\exp\left(\frac{p\log(p)-\delta}{1-p}\right),p+\sqrt{\delta/2}\right)$
\ENDIF
\end{algorithmic}
\end{algorithm}

\section{Another Approximation Method for kl-UCB}\label{app:originapprox}
In this section, we show an approximation of the KL divergence $d_{kl}$. Then we design an algorithm that solves the problem $P_1(d_{kl})$ efficiently for kl-UCB. 

Recall that $p\in\Theta$ and $\delta>0$ are the inputs of the problem $P_1(d_{kl})$. Given any approximation error $\epsilon>0$, let $q_k=\exp(-k\epsilon/p)\in\Theta$ for any ${k\geq0}$. Then, there exists an integer $L(p)=\left\lfloor\frac{-p\log(p)}{\epsilon}\right\rfloor$ such that $q_k\geq p$ if and only if $k\leq L(p)$. For each $k$, we construct a corresponding function, which is a generalization of $d_{lb}$, such that $d_{lb}^k(p,q)=$
\begin{eqnarray}\label{eqn:hellinger}
\begin{cases}
p\log\left(\frac{p}{q_k}\right)+(1-p)\log\left(\frac{1-p}{1-q}\right), & q\leq q_k \cr
p\log\left(\frac{p}{q_k}\right)+(1-p)\log\left(\frac{1-p}{1-q_k}\right), & q>q_k
\end{cases}.
\end{eqnarray}
Note that $d_{lb}=d_{lb}^0$ since $q_0=1$. The following result shows that $\max\limits_{k\leq L(p)}d_{lb}^k$ is an $\epsilon$-approximation of the function $d_{kl}$ on the interval $[p,1]$. The proof is presented in Section \ref{app:approx}

\begin{proposition}\label{prop:approx}
Given $p\in\Theta$ and $\epsilon >0$, for any $q\in[p,1]$, we have that
\begin{equation}
0\leq d_{kl}(p,q)-\max\limits_{k\leq L(p)}d_{lb}^k(p,q)\leq \epsilon.
\end{equation}
\end{proposition}

In stead of solving $P_1(d_{kl})$, we can solve $P_1\left(\max\limits_{k\leq L(p)}d^k_{lb}\right)$ to obtain an $\epsilon$-optimal solution of $P_1(d_{kl})$. Note that the problem $P_1(d^k_{lb})$ has a closed-form solution. By the trick we use in Section \ref{sec:boosting}, it is equivalent to solve $\min\limits_{k\leq L(p)}P_1\left(d^k_{lb}\right)$ with $L(p)$ computational complexity. However, due to the structure of $d^k_{lb}$, we can use a bisection search to reduce the complexity to $\log (L(p))$. The approximation-based bisection search method is presented in Algorithm \ref{alg:approx}. The guarantee of Algorithm \ref{alg:approx} is presented in the following result, of which the proof is presented in Section \ref{app:thmapprox}.

\begin{theorem}\label{theorem:thmapprox}
Let $q^*$ be the optimal solution of the problem $P_1(d_{kl})$. Given any $\epsilon>0$, Algorithm \ref{alg:approx} returns an $\epsilon$-opitmal solution $q\prime$ of the problem $P_1(d_{kl})$ such that $q^\prime\geq q^*$ and $0\leq d_{kl}(p,q\prime)-d_{kl}(p,q^*)\leq\epsilon$. The computational complexity of Algorithm \ref{alg:approx} is $\log\left(\frac{-p\log(p)}{\epsilon}\right)$, which is at most $\log\left(\frac{1}{e\epsilon}\right)$ and $e$ is the natural number.
\end{theorem}

\begin{remark}
Applying the bisection search for the optimal $q^*$ of $P_1(d_{kl})$ within the interval $[p,1]$ can find a solution $q^\prime$ such that $|q\prime-q^*|\leq \epsilon$ with the computational complexity of $\log\left(\frac{1-p}{\epsilon}\right)$ iterations. However, our approximation method has two advantages. On one hand, $-p\log(p)\leq 1-p$ holds for any $p\in\Theta$, which implies that our approximation method enjoys lower complexity. On the other hand, the gap $|d_{kl}(p,q^*)-d_{kl}(p,q^\prime)|$ is unbounded even though $|q\prime-q^*|\leq \epsilon$ while our approximation method guarantees bounded KL divergence gap.
\end{remark}

\begin{algorithm}[t]
\caption{Approximation-Based Bisection Search}
\label{alg:approx}
\begin{algorithmic}
\REQUIRE empirical mean $p$, exploration bonus $\delta$ and approximation error $\epsilon$
\STATE{Initialization: $L=\left\lfloor\frac{-p\log(p)}{\epsilon}\right\rfloor$ and $k=\left\lfloor L/2\right\rfloor$}
\WHILE{\TRUE}
\IF{$d_{kl}(p,q_{k+1})<\delta\leq d_{kl}(p,q_k)$}
\RETURN $1-(1-p)\exp\left(\frac{p\log(p)-\delta+k\epsilon}{1-p}\right)$
\ELSIF{$\delta>d_{kl}(p,q_k)$}\STATE{$k=\lfloor k/2\rfloor$}\ELSE\STATE{$k=\lfloor 3k/2\rfloor$}\ENDIF
\ENDWHILE
\end{algorithmic}
\end{algorithm}

\subsection{Proof of Proposition \ref{prop:approx}}\label{app:approx}
\begin{proof}
For any $q\in[p,1]$, there exists $k\leq L(p)$ such that $q_{k+1}<q\leq q_k$. Thus, we have that
\begin{align}
d_{kl}(p,q)-d_{lb}^k(p,q)&=p\log\left(\frac{p}{q}\right)-p\log\left(\frac{p}{q_k}\right)\\
&=p\log\left(\frac{q_k}{q}\right)\\
&\leq p\log\left(\frac{q_k}{q_{k+1}}\right)\\
&=\epsilon.
\end{align}
Hence, we have that
\begin{equation}
d_{kl}(p,q)-\max\limits_{i\leq L(p)}d_{lb}^i(p,q)\leq d_{kl}(p,q)-d_{lb}^k(p,q)\leq \epsilon.
\end{equation}
For any $k\leq L(p)$, we claim that $d_{lb}^k(p,q)\leq d_{kl}(p,q)$. If $q\leq q_k$, then we have that 
\begin{equation}
d^k_{lb}(p,q)=p\log\left(\frac{p}{q_k}\right)+(1-p)\log\left(\frac{1-p}{1-q}\right)\leq p\log\left(\frac{p}{q}\right)+(1-p)\log\left(\frac{1-p}{1-q}\right)=d_{kl}(p,q).
\end{equation}
If $q>q_k$, by $q_k\geq p$ we have that
\begin{equation}
d^k_{lb}(p,q)=p\log\left(\frac{p}{q_k}\right)+(1-p)\log\left(\frac{1-p}{1-q_k}\right)\leq p\log\left(\frac{p}{q}\right)+(1-p)\log\left(\frac{1-p}{1-q}\right)=d_{kl}(p,q).
\end{equation}
Thus, we have that $d_{lb}^k(p,q)\leq d_{kl}(p,q)$
Hence, we have that $d_{kl}(p,q)-\max\limits_{i\leq L(p)}d_{lb}^i(p,q)\geq 0$.
\end{proof}

\subsection{Proof of Theorem \ref{theorem:thmapprox}}\label{app:thmapprox}
\begin{proof}
Let $q^*$ be the optimal solution of the problem $P_1(d_{kl})$. It is clear that $q^*\in[p,1]$. Let $q\prime$ be the output of Algorithm \ref{alg:approx}. We first claim that Algorithm \ref{alg:approx} is a bisection search method to solve the problem $P_1\left(\max\limits_{k\leq L(p)}d^k_{lb}\right)$. By Proposition \ref{prop:approx}, we have that $q^\prime\geq q^*$ and $0\leq d_{kl}(p,q\prime)-d_{kl}(p,q^*)\leq\epsilon$.  The computational complexity of Algorithm \ref{alg:approx} is $\log\left(L(p)\right)$, which is at most $\log\left(\frac{1}{e\epsilon}\right)$ since $-p\log(p)\leq \frac{1}{e}$ for any $p\in\Theta$. Now, it remains to show how bisection search works for solving $P_1\left(\max\limits_{k\leq L(p)}d^k_{lb}\right)$.

Given $p\in\Theta$, the set of points $\{0,d_{kl}(p,q_{L(p)}),\ldots,d_{kl}(p,q_{0})\}$ is a partition of the extended interval $[0,\infty]$. Note that $d_{kl}(p,q_{0})=d_{kl}(p,1)=\infty$ for convention. Given $\delta>0$, there exists $k\leq L(p)$ such that either $d_{kl}(p,q_{k+1})<\delta\leq d_{kl}(p,q_k)$ and $k<L(p)$ or $0<\delta\leq d_{kl}(p,q_{k})$ and $k=L(p)$. Recall that $q^\prime$ is the solution of the problem $P_1\left(\max\limits_{k\leq L(p)}d^k_{lb}\right)$. Thus, we have that $\max_{i\leq L(p)}d^i_{lb}(p,q^\prime)=\delta$. For any integer $i$ such that $k<i\leq L(p)$, we have that $d_{lb}^i(p,q)\leq d_{kl}(p,q_i)\leq d_{kl}(p,q_{k+1})<\delta$ for any $q\in\Theta$. For any integer $i$ such that $0\leq i<k$, we have that 
\begin{equation}
d^i_{lb}(p,q^\prime)=p\log\left(\frac{p}{q_i}\right)+(1-p)\log\left(\frac{1-p}{1-q^\prime}\right)\leq d^k_{lb}(p,q^\prime).
\end{equation}
Hence, we have that $\max_{i\leq L(p)}d^i_{lb}(p,q^\prime)=d^k_{lb}(p,q^\prime)$. So Algorithm \ref{alg:approx} uses bisection search for the $k$ and returns the solution of $d^k_{lb}(p,q)=\delta$.
\end{proof}


\section{Dual Method for KL-UCB}\label{sec:dual}
KL-UCB algorithm, proposed by \citeauthor{cappe2013kullback}~\shortcite{cappe2013kullback}, is the generalization of kl-UCB to the case when the distribution $v_a$ is arbitrary on $\Theta$. It has been shown to be optimal in the case of general distributions with bounded support. KL-UCB replaces the $P_1(d_{kl})$ of kl-UCB by the following problem:
\begin{align}
P_2: \max_{\bq\in S_n}~~& \sum_{i=1}^n \alpha_iq_i\\
s.t. ~~&d_{KL}(\bp,\bq)=\sum_{i=1}^np_i\log\left(\frac{p_i}{q_i}\right)\leq \delta,
\end{align}
where the set $\balpha=\{\alpha_1,\ldots,\alpha_n\}$ is the union of the empirical support and $\{1\}$, $\bp=(p_1,\ldots,p_n)$ is the corresponding empirical distribution and $S_n$ is the simplex in $\mathcal{R}^n$. The problem $P_2$ is the generalization of $P_1(d_{kl})$ to the general case. Given empirical distribution $\bp$ over the support $\balpha$, the upper confidence bound is the largest expected mean of a distribution $\bq^*$ over the $\balpha$ such that the KL-divergence between $\bp$ and $\bq^*$ is at most the exploration bonus $\delta$. Note that $\delta=\left(\log(t)+c\log(\log(t))\right)/N_a(t)$ for each arm $a$. 

Without loss of generality, we assume that $0\leq\alpha_1<\cdots<\alpha_{n-1}<\alpha_n=1$. Thus, we have that $p_i>0$ for $1\leq i\leq n-1$. Let $l=\alpha_n$ if $p_n>0$ and $l=\alpha_{n-1}$ if $p_n=0$. Then we define the function $f:(l,\infty)\to (0,\infty)$ such that
\begin{equation}
f(\lambda)=\sum_{i=1}^np_i\log(\lambda-\alpha_i)+\log\left(\sum_{i=1}^n\frac{p_i}{\lambda-\alpha_i}\right).
\end{equation}

Observe that $P_2$ is a linear program under convex constraints. By the Lagrangian dual method and the Karush-Kuhn-Tucker conditions, we have the following result.
\begin{lemma}\label{lem:lagrangian}
\emph{(Algorithm 1 in  \cite{cappe2013supplement})}
Let $\bq^*$ be the optimal solution of $P_2$. If $p_n=0$ and $f(1)<\delta$, then
\begin{align}\nonumber
q^*_i&=\exp(f(1)-\delta)\frac{p_i/(1-\alpha_i)}{\sum_{j=1}^{n-1}p_j/(1-\alpha_j)}, \forall i\leq n-1\\
q^*_n&=1-\exp(f(1)-\delta).
\end{align}
Otherwise,
\begin{align}
q^*_i&=\frac{p_i/(\lambda^*-\alpha_i)}{\sum_{j=1}^{n}p_j/(\lambda^*-\alpha_j)}, \forall i\leq n
\end{align}
where $\lambda^*$ is the root of the equation $f(\lambda)=\delta$.
\end{lemma}
Lemma \ref{lem:lagrangian} shows that the $n$-dimensional convex optimization problem $P_2$ can be reduced to finding the root of $f(\lambda)=\delta$. Lemma \ref{lem:lagrangian} is also found by \cite{filippi2010optimism}~\shortcite{filippi2010optimism}.

\begin{lemma}\label{lem:fbound}
We have that
$f(\lambda)\leq \frac{(l-\alpha_1)^2}{8(\lambda-l)^2}$ holds for any $\lambda\in(l,\infty)$.
\end{lemma}
The proof is presented in Section \ref{app:fbound}.
\begin{proposition}\label{prop:fsearch}
Given any $\epsilon>0$, the complexity of finding the root of $f(\lambda)=\delta$ by bisection search within the interval $\left[l,l+\frac{l-\alpha_1}{2\sqrt{2\delta}}\right]$ is at most $O\left(\log\left(\frac{1}{\epsilon\sqrt{\delta}}\right)\right)$ iterations. Note that the complexity for each iteration is $O(n)$ computations.
\end{proposition}
The proof is presented in Section \ref{app:fsearch}.

\subsection{Proof of Lemma \ref{lem:fbound}}\label{app:fbound}
\begin{proof}
We consider the case that $p_n>0$. The proof of the case that $p_n=0$ follows similarly by reducing the problem to $n-1$ dimension. Then we have $l=\alpha_n$ and the function $f(\lambda)$ is well-defined on $(\alpha_n,\infty)$. 

Let $a_i=\frac{\sqrt{p_i}}{\lambda-\alpha_i}$ and $b_i=\sqrt{p_i}$. Since $p_i>0$ for any $i\leq n$ and $\lambda\in(\alpha_n,\infty)$, we have that $a_i>0$ and $b_i>0$ for any $i\leq n$. By $0\leq\alpha_1<\cdots<\alpha_{n-1}<\alpha_n=1$, we have that $0<\frac{1}{\lambda-\alpha_1}=m\leq\frac{a_i}{b_i}=\frac{1}{\lambda-\alpha_i}\leq M=\frac{1}{\lambda-\alpha_n}<\infty$. By Cassel's inequality, we have that
\begin{equation}\label{eqn:cassel}
\sum_{i=1}^n\left(\frac{\sqrt{p_i}}{\lambda-\alpha_i}\right)^2\cdot\sum_{i=1}^n\left(\sqrt{p_i}\right)^2\leq \frac{(M+m)^2}{4Mm}\left(\sum_{i=1}^n\frac{p_i}{\lambda-\alpha_i}\right)^2
\end{equation}

We define the function $g(\lambda)=f(\lambda)-\frac{(\alpha_n-\alpha_1)^2}{8(\lambda-\alpha_n)^2}$. By taking the derivative, we have that
\begin{align}
g^\prime(\lambda)&=\sum_{i=1}^n\frac{p_i}{\lambda-\alpha_i}-\frac{\sum_{i=1}^n\frac{p_i}{(\lambda-\alpha_i)^2}}{\sum_{i=1}^n\frac{p_i}{\lambda-\alpha_i}}+\frac{(\alpha_n-\alpha_1)^2}{4(\lambda-\alpha_n)^3}\\
&=\left(\sum_{i=1}^n\frac{p_i}{\lambda-\alpha_i}\right)^{-1}\left[\left(\sum_{i=1}^n\frac{p_i}{\lambda-\alpha_i}\right)^2-\sum_{i=1}^n\left(\frac{\sqrt{p_i}}{\lambda-\alpha_i}\right)^2\cdot\sum_{i=1}^n\left(\sqrt{p_i}\right)^2\right]+\frac{(\alpha_n-\alpha_1)^2}{4(\lambda-\alpha_n)^3}\\
&\geq\left(\sum_{i=1}^n\frac{p_i}{\lambda-\alpha_i}\right)\left[1-\frac{(M+m)^2}{4Mm}\right]+\frac{(\alpha_n-\alpha_1)^2}{4(\lambda-\alpha_n)^3}~~~~~~~~~~~~~ \text{(follows by (\ref{eqn:cassel}))}\\
&=-\left(\sum_{i=1}^n\frac{p_i}{\lambda-\alpha_i}\right)\frac{(\alpha_n-\alpha_1)^2}{4(\lambda-\alpha_n)(\lambda-\alpha_1)}+\frac{(\alpha_n-\alpha_1)^2}{4(\lambda-\alpha_n)^3}\\
&\geq-\left(\sum_{i=1}^n\frac{p_i}{\lambda-\alpha_n}\right)\frac{(\alpha_n-\alpha_1)^2}{4(\lambda-\alpha_n)(\lambda-\alpha_n)}+\frac{(\alpha_n-\alpha_1)^2}{4(\lambda-\alpha_n)^3}\\
&=0.
\end{align}
Thus, we have that the function $g$ is non-decreasing on the interval $(\alpha_n,\infty)$. By taking the limit, we have that 
\begin{equation}
\lim_{\lambda\to\infty}g(\lambda)=0.
\end{equation}
Hence, we have that $g(\lambda)\leq 0$.
\end{proof}

\subsection{Proof of Proposition \ref{prop:fsearch}}\label{app:fsearch}
\begin{proof}
Let $\lambda^*$ be the root of $f(\lambda)=\delta$. Thus, we have that $\delta=f(\lambda^*)\leq \frac{(l-\alpha_1)^2}{8(\lambda^*-l)^2}$. So we have that $\lambda^*\leq l+\frac{l-\alpha_1}{2\sqrt{2\delta}}$. We apply bisection search on the interval $[l,l+\frac{l-\alpha_1}{2\sqrt{2\delta}}]$ and the complexity is $\log\left(\frac{\frac{l-\alpha_1}{2\sqrt{2\delta}}}{\epsilon}\right)=O(\log(\frac{1}{\epsilon\sqrt{\delta}}))$ iterations. Note that each iteration needs to compute $f(\lambda)$ with $O(n)$ computations.
\end{proof}

\end{document}